\documentclass[twoside]{amsart}
\usepackage[utf8]{inputenc}
\usepackage{amsmath}
\usepackage{amsfonts}
\usepackage{amssymb}
\usepackage{graphicx}
\usepackage{color}
\usepackage{setspace}
\usepackage[normalem]{ulem}
\usepackage{hyperref}
\usepackage{graphicx}
\usepackage{epstopdf}
\usepackage{float}
\usepackage{overpic}
\usepackage[left=2cm,right=2cm,top=2cm,bottom=2cm]{geometry}
\def\hindu{\arabic}

\newtheorem{theorem}{Theorem}[section]
\newtheorem{lemma}{Lemma}[section]

\newtheorem{rem}{Remark}[section]
\newtheorem{definition}{Definition}[section]
\newtheorem{proposition}{Proposition}[section]

\renewcommand{\theequation}{\hindu{section}.\hindu{equation}}
\def\bhag#1{\noindent
\setcounter{equation}{0}
\pagestyle{plain}
\section{#1}
}

\def\NN{{\mathbb N}}
\def\RR{{\mathbb R}}
\def\CC{{\mathbb C}}

\def\ZZ{{\mathbb Z}}

\def\PPI{{{\rm I}\kern-1pt\Pi}}

\def\b #1;{{\bf #1}}
\def\x{{\bf x}}
\def\k{{\bf k}}
\def\y{{\bf y}}
\def\r{{\bf r}}
\def\v{{\bf v}}

\def\w{{\bf w}}
\def\z{{\bf z}}

\def\m{\mathfrak{m}}
\def\j{\mathbf{j}}

\def\C{{\mathcal C}}

\def\esssup{\mathop{\hbox{{\rm ess sup}}}}

\def\be{\begin{equation}}
\def\ee{\end{equation}}
\def\bea{\begin{eqnarray}}
\def\eea{\end{eqnarray}}

\def\disp{\displaystyle}

\def\donchitre#1#2{\vskip 6.5cm\noindent
\parbox[t]{1in}{\special{eps:#1.eps x=6.5cm y=5.5cm}}
\hbox to 7cm{}\parbox[t]{0.0cm}{\special{eps:#2.eps x=6.5cm y=5.5cm}}}

\def\BB{{\mathbb B}}

\def\bs#1{{\boldsymbol{#1}}}

\begin{document}
\title{Encoding of data sets and algorithms}

\author{Katarina Doctor$^1$}
\thanks{
$^1$Navy Center for Applied Research in AI, Information Techchology Division, U.S. Naval Research Laboratory, Washington DC 20375.
\email{katarina.doctor@nrl.navy.mil}}

\author{Tong Mao$^2$}
\thanks{
$^2$Institute of Mathematical Sciences, Claremont Graduate University (United States). The research of this author was funded by NSF DMS grant 2012355.
\email{tong.mao@cgu.edu}}

\author{Hrushikesh Mhaskar$^3$}
\thanks{
$^3$Institute of Mathematical Sciences, Claremont Graduate University (United States). The research of HNM was supported in part by ARO grant W911NF2110218, NSF DMS grant 2012355, and a Faculty Visiting Fellowship program at ONR 
\email{hrushikesh.mhaskar@cgu.edu}} 

\date{}

\begin{abstract}
In many high-impact applications, it is important  to ensure the quality of output of a machine learning algorithm as well as its reliability in comparison with the complexity of the algorithm used.
In this paper, we have initiated a mathematically rigorous theory to decide which models (algorithms applied on data sets) are close to each other in terms of certain metrics, such as performance and the complexity level of the algorithm.
This involves creating a grid on the hypothetical spaces of data sets and algorithms so as to identify a finite set of probability distributions from which the data sets are sampled and a finite set of algorithms. A given threshold metric acting on this grid will express the nearness (or statistical distance) from each algorithm and data set of interest to any given application. 
A technically difficult part of this project is to estimate the so-called metric entropy of a compact subset of functions of \textbf{infinitely many variables} that arise in the  definition of these spaces.
\end{abstract}

\maketitle

\bhag{Introduction}\label{bhag:intro}
In many high-impact  applications  of  machine learning,  the data is limited and training is challenging. For these applications, it is desirable to have predictions with the highest assurances from the available data while minimizing uncertainty. 
In particular,  it is important  to ensure the quality of output of a machine learning algorithm as well as its reliability in comparison with the complexity of the algorithm used.
The objective of this work is to develop a systematic and mathematically rigorous approach to decide what is the complexity level of algorithm that is sufficient on the task domain to output the desired performance, reliability, and uncertainty. 

One attractive idea in this context is that of Rashomon curves \cite{semenova2019study, fisher2019all}. 
The question is the following: if one finds that an algorithm with a certain complexity level works well on a task on a data set, are there likely to be simpler algorithms that will also work within a certain tolerance of this algorithm? 
More generally, which class of algorithms can be expected to behave similarly on which kind of data sets?
Unfortunately, there seems to be no mathematically precise formulation of this problem.
Our purpose in this paper is to initiate such a rigorous study.

Intuitively, we wish to obtain a grid on the set of data sets and algorithms, i.e., a finite set of data sets and algorithms so that for every algorithm of interest on every data set of interest, there is some point on the grid that is close to the data set and algorithm, as measured by some parameters.    In order to make this more precise, we clarify what the terms "data sets" and "algorithms" mean for our purposes.

We will assume that each data set is a random sample from an unknown probability distribution on a domain. To be precise, we assume that each distribution is supported on some compact subset of an ambient Euclidean space of dimension $q$, without loss of generality, on $[-1,1]^q$.
Of course, different samples may come from the same distribution, in which case there is no theoretical difference between two such data sets. 
On the other hand, problems of sample bias are sometimes dealt with by omitting some of the components from each of these samples.
Naturally, the resulting data has a different distribution, so the reduced data set is considered in this paper to be a different data set from the original.

In view of the Riesz representation theorem and the Banach-Alaoglu theorem, the set of all probability measures is a compact subset of the dual space $C([-1,1])^q)^*$. 
This set is an unmanageably large set representing \emph{every possible} data set that could possibly arise.
 We model the set of data sets of interest by a smaller compact subset $\mathbb{P}$ of the dual space $C([-1,1])^q)^*$.

A lucid description of the meaning of the term "algorithm" and a precise mathematical definition of the term can be found in \cite[Section~1.1]{knuth1997art}.
An algorithm is a function from the input space (the data set) to the output space (real numbers, class labels, etc.) with some additional properties.
As in the notion of Rashomon sets as explained in \cite{semenova2019study, fisher2019all},  one is not interested in the actual algorithms themselves but more in how they perform  different tasks on data sets with respect to certain parameters such as stability, accuracy, complexity level of the algorithms, etc.
It is unlikely that two algorithms will match in terms of all these parameters for all the data sets in question. However, if there are two algorithms (or network architectures with different complexity levels) that lead to the same measurements of these quantities, then there is no need to distinguish between these.
The stability of an algorithm should mean that when two data sets (meaning two  probability distributions) are ``close by,'' then the accuracy and complexity of the algorithm on the two data sets should be close as well.
This is captured by a notion of smoothness of the algorithms considered as functions on the data sets.

We assume a set $\mathbb{A}$ of algorithms that act on each data set in $\mathbb{P}$. 
Each of these algorithm gives rise to a certain number $m$ of parameters.
We are thus interested in a mapping  $F^*: \mathbb{P}\times 
\mathbb{A}\to\RR^m$. 
Without loss of generality, we may assume $m=1$ in this paper.
This is represented in Figure~\ref{fig:basicidea}.
\begin{figure}[ht]
\begin{center}
\includegraphics[width=0.8\textwidth]{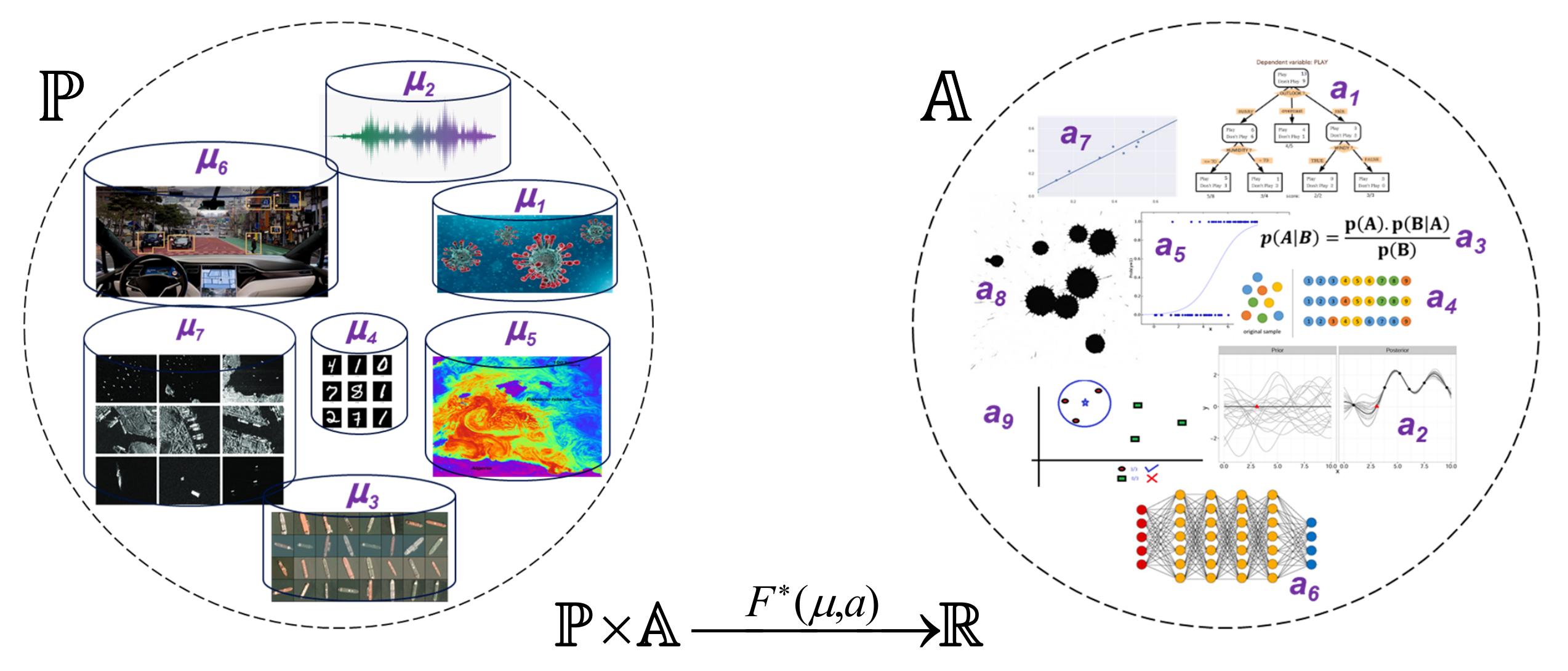} 
\end{center}
\caption{Different data sets may come from the same probability distribution on a domain. We consider $\mathbb{P}$ to be a compact subset of such distributions. $\mathbb{A}$ is the set of algorithms of interest to us, and $F^*$ is the function that maps a given probability distribution and an algorithm to an $m$-dimensional vector of quantities of interest.}
\label{fig:basicidea}
\end{figure}

We don't expect two algorithms to agree on all the data sets with respect to all of these parameters, i.e., we are assuming that if  $a_1,a_2\in\mathbb{A}$ and
$$
F^*(\mu,a_1)=F^*(\mu,a_2)\quad\mbox{ for all } \mu\in\mathbb{P} \Rightarrow a_1=a_2.
$$
This means that every $a\in\mathbb{A}$ corresponds to a \textbf{unique} mapping $F_a$ on $\mathbb{P}$ defined by 
\be\label{eq:algmap}
F_a(\mu)=F^*(\mu,a), \qquad \mu\in\mathbb{P}.
\ee

An algorithm $a\in\mathbb{A}$ is defined to be \emph{stable} if $F_a$ is a continuous function on $\mathbb{P}$ with a properly defined topology on $\mathbb{P}$.

These considerations prompt us to consider a set $\mathfrak{X}$ of continuous functions from $\mathbb{P}$ to $\RR^m$. 
We will assume implicitly that to every element  $F\in\mathfrak{X}$ corresponds a (necessarily unique) algorithm $a\in \mathbb{A}$ such that $F=F_a$ as defined in \eqref{eq:algmap}.
We will then abuse the notation and refer to $F\in\mathfrak{X}$ as an algorithm.

In this paper, we will assume both $\mathbb{P}$ and $\mathfrak{X}$ to be compact metric spaces with appropriate metrics. 
In fact, in view of the Ascoli theorem, $\mathfrak{X}$ is then an equicontinuous family of functions on $\mathbb{P}$.
We then fix a ``tolerance'' $\epsilon>0$, and find $\epsilon$-nets $\mathbb{P}_\epsilon$ and $\mathfrak{X}_\epsilon$ for $\mathbb{P}$ and $\mathfrak{X}$\footnote{If $K$ is a compact subsset of a metric space $X$ and $\epsilon>0$, then a finite set $K_\epsilon\subset X$ is called an $\epsilon$-net for $K$ if $K$ is covered by balls of radius $\epsilon$ centered at points in $K_\epsilon$.}, respectively.
Then $\mathbb{P}_\epsilon\times \mathfrak{X}_\epsilon$ is an $\epsilon$-net for $\mathbb{P}\times\mathfrak{X}$. 
For any data set $\mu\in\mathbb{P}$ and $F\in \mathfrak{X}$ (equivalently, an algorithm $a\in \mathbb{A}$), there is $\mu_1\in 
\mathbb{P}_\epsilon$ and $F_1\in \mathfrak{X}_\epsilon$ (equivalently, an algorithm $a_1$) such that the behavior of $a$ on $\mu$ is $\epsilon$-similar to the behavior of $a_1$ on $\mu_1$. 
Thus, the problem reduces to finding a minimal $\epsilon$-net for $\mathbb{P}\times \mathfrak{X}$ (or, with our identification of the space $\mathbb{A}$ of algorithms with $\mathfrak{X}$, $\mathbb{P}\times \mathbb{A})$) as represented in Figure~\ref{fig:basicidea}.
\begin{figure}[ht]
\begin{center}
\includegraphics[width=0.6\textwidth]{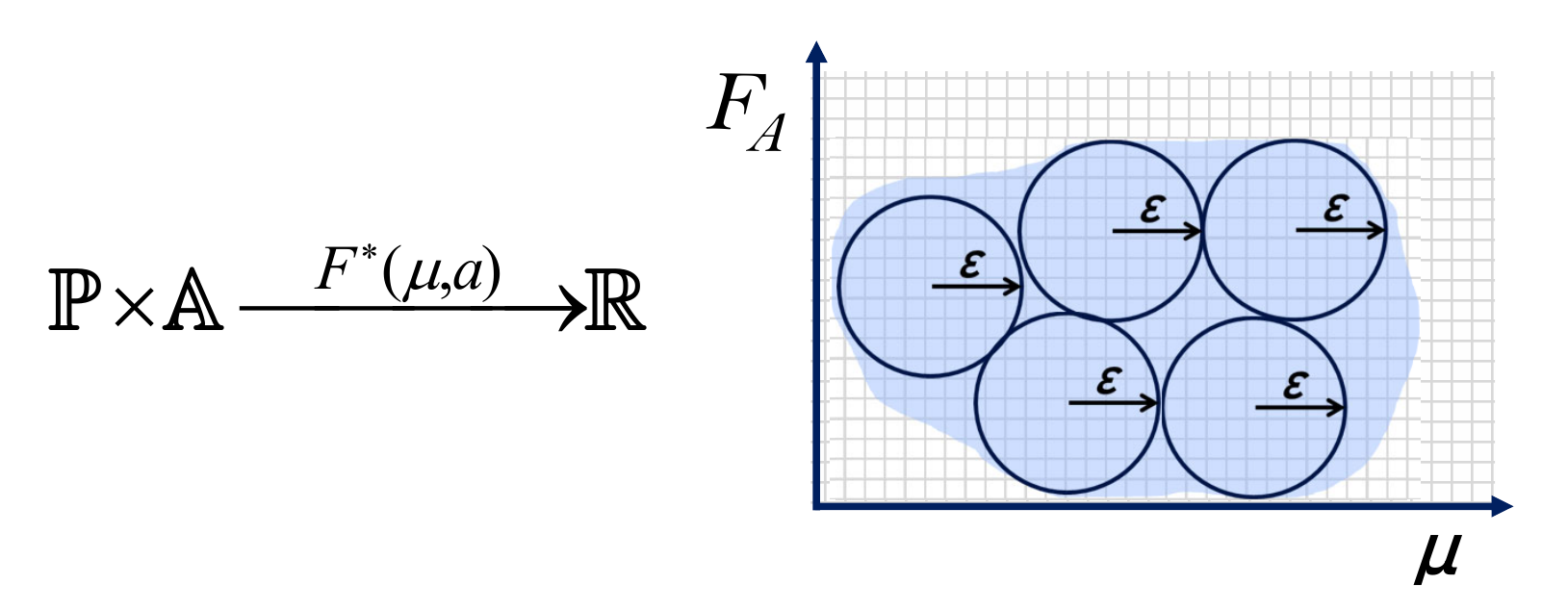} 
\end{center}
\caption{The $\epsilon$-net for the set $\mathbb{P} \times\mathbb{A}$, where $\mathbb{A}$ is identified with a set of functionals on $\mathbb{P}$.  The problem is to estimate the minimal number of balls of radius $\epsilon$ to cover the set; the challenge being the fact that $\mathbb{A}$ is a set of functionals acting on an infinite dimensional space $\mathbb{P}$.}
\label{fig:epsilonRadiusBalls}
\end{figure}

The major technical difficulty here is that $\mathfrak{X}$ is a set of functions on an infinite dimensional space rather than a finite dimensional Euclidean space as is usual in common machine learning problems.
A simplistic idea is to obtain a finite set of parameters for the probability distributions and to treat $\mathfrak{X}$ a set of functions on these. For example, if $\mathbb{P}$ were a set of normal distributions, then the means and standard deviations would describe this set completely.
However, in practice, the distributions are not prescribed in terms of finitely many parameters. 
Indeed, a central technical challenge in  machine learning is that the distributions involved are unknown; in particular, one needs non-parametric methods to deal with these.

It is still possible to restrict ourselves to those distributions that have a smooth density function. 
In turn, this function can be expanded in an orthogonal series, such as a multivariate tensor product Chebyshev polynomial expansion, and the coefficients of this expansion can be used as the parameters for the distribution.
If the density functions are smooth enough, then finitely many low-order coefficients will approximate the density well enough, and elements of $\mathfrak{X}$ can be thought of as functions of these low-order coefficients.

Although this simple idea does reduce the problem to the case of functions on a Euclidean space, there is still a technical problem.
In order to get a good approximation to the density, one needs a large number of coefficients.
The curse of dimensionality then poses a big challenge, requiring much more detailed analysis than what is available in the literature.

The organization of this paper is as follows. In Section~\ref{bhag:basicconcept}, we review the basic concepts of entropy, analytic and entire functions. 
Our main results are stated in Section~\ref{bhag:main}, where we develop an abstract framework, which is then applied to get the estimates on entropies for certain classes of analytic and entire functions, culminating the estimates for a class of functionals. 
In Section~\ref{bhag:computeissue}, we discuss some ideas on how to generate computationally some classes of analytic and entire  functions, as well as  $\epsilon$-nets for a finite dimensional ellipsoids, which form a  theoretical backbone for our estimates.
The proofs of the results in Section~\ref{bhag:main} are given in Section~\ref{bhag:proofs}. 
For the convenience of the reader, we include an appendix, in which we prove certain estimates on the approximation of analytic and entire functions which motivate our definition of the classes defined in Section~\ref{bhag:main}.

\section{Basic concepts}\label{bhag:basicconcept}

In this section, we explain the basic concepts used in this paper. 
Section~\ref{sec:notation} describes the multivariate notation.
Section~\ref{sec:entropy} summarizes the definition of metric entropy and capacity related to the minimal number of balls of a given radius to cover a compact set.
The probability measures to be studied have densities that are analytic, while the functionals are entire functions of finite type defined on an infinite dimensional sequence space. 
These ideas are described in Section~\ref{sec:analytic entire}.
Section~\ref{sec:chebyshev} reviews certain basic notions regarding multivariate Chebyshev polynomials which are used  to encode both analytic and entire functions.

\subsection{Multivariate notation}\label{sec:notation}
In the sequel, we denote by $d\in\NN\cup\{\infty\}$ a generic dimension. 
Vectors will be denoted by boldface letters, e.g., $\x=(x_1,\cdots,x_d)\in\RR^d$. 
The symbol $|\x|_p$ will denote the $\ell^p$ norm of the vector $\x$.
Binary operations among vectors are meant to be in componentwise sense; e.g., $\x\y=(x_1y_1,\cdots,x_dy_d)$, $\x^\y = \prod_{j=1}^d x_j^{y_j}$, $\x/\y=(x_1/y_1,\cdots,x_d/y_d)$. 
Similarly, $\x<\y$ means $x_j<y_j$ for $j=1,\cdots,d$, etc.
The inner product between two vectors $\x,\y$ is denoted by $\x\cdot\y$.
For $r>0$, we write $I_r=[-r,r]$, and for a vector $\mathbf{r}$, $I_{\mathbf{r}}=\prod_{j=1}^d [-r_j,r_j]$. 
Finally, $I=I_1$.
For $0<\rho<1$, the ellipse $U_\rho$ is defined by
$$
U_\rho=\left\{z\in\CC : |z+\sqrt{z^2-1}|<1/\rho\right\},
$$
where the principal branch of the square root is chosen. 
With the Joukowski transformation $w=z+\sqrt{z^2-1}$, $U_\rho$ is mapped onto the disc  $\Gamma_\rho=\left\{w\in\CC : |w|<1/\rho\right\}$.

Let $0<\bs\rho=(\rho_1,\dots,\rho_d)<1$, the poly-ellipse $U_{\bs\rho}$(respectively, the poly-disc $\Gamma_{\bs\rho}$) are defined by $U_{\bs\rho}=\prod_{j=1}^d U_{\rho_j}$ (respectively, $\Gamma_{\bs\rho}=\prod_{j=1}^d \Gamma_{\rho_j}$). 
When $\bs\rho=(\rho,\cdots,\rho)$, we will abuse the notation and write $U_{\rho,d}=U_{(\rho,\dots,\rho)}$. If the dimension is clear in the context, we drop the subscript $d$ and write $U_\rho=U_{(\rho,\dots,\rho)}$.
Similar conventions are adopted also for the poly-discs and rectangular cells.

\subsection{Entropy and Capacity}\label{sec:entropy}

The material in this section is based on \cite[Chapter 15]{lorentz_advanced}.

Let $(X,\|\cdot\|)$ be a normed linear space, $K\subset X$, and $\epsilon>0$ be given.
\begin{itemize}
    \item[(a)] A set $\hat K\subset X$ is called an \emph{$\epsilon$-net} for $K$ if, for each $x\in K$, there is at least one $y\in\hat K$ such that $\|x-y\|\leq\epsilon$.
    \item[(b)] Points $y_1,\dots,y_m\in K$ are called \emph{$\epsilon$-separable} if
    $$\|y_i-y_j\|\geq\epsilon,\quad i\neq j.$$
\end{itemize}

\begin{definition}\label{def:entropydef}
Let $(X,\|\cdot\|)$ be a normed linear space, $K\subset X$ is compact. For any $\epsilon>0$, let $ \mathfrak N_\epsilon(K,\|\cdot\|)$ be the minimal value of $n$ such that there exists an $\epsilon$-net for $K$ consisting of $n$ points. The \textbf{entropy} of $K$ is defined as
\begin{equation}\label{eq:entropydef}
    H_\epsilon(K,\|\cdot\|)=\log \mathfrak{N}_\epsilon(K,\|\cdot\|).
\end{equation}
Let $ \mathfrak M_\epsilon(K,\|\cdot\|$ be the maximal value of $m$ for which there exists $m$ $\epsilon$-separable points for $K$. The \textbf{capacity} of $K$ is defined as
\begin{equation}\label{eq:capacitydef}
    C_\epsilon(K,\|\cdot\|)=\log \mathfrak{M}_\epsilon(K,\|\cdot\|).
\end{equation}

\end{definition}

The connection between capacity and metric entropy is given in the following proposition.
\begin{proposition}\label{prop:capacity}
Let $X$ be a normed linear space.
 For each compact set $K\subset X$ and each $\epsilon>0$,
    \begin{equation}
        C_{2\epsilon}(K,\|\cdot\|)\leq H_\epsilon(K,\|\cdot\|)\leq C_\epsilon(K,\|\cdot\|).
    \end{equation}
\end{proposition}

\subsection{Analytic  and entire functions}\label{sec:analytic entire}

\begin{definition}[Analytic functions]
    Let $q\in\NN$, $\rho>0$, $f$ is said to be an \textbf{analytic function} on $U_\rho:=\{\z\in\CC^q:\ \left|z_j+\sqrt{z_j^2-1}\right|<1/\rho,\ j=1,\dots,q\}$ if it is complex differentiable at each $\z\in U_\rho$.
\end{definition}

\begin{definition}\label{def:entirefn}{\rm [Entire functions of exponential type]}
 {\rm (a)}   Let $Q\in\NN$, $\tau>0$.
A function $F :\CC^Q\to\CC$ is called an \textbf{entire function of exponential type $\tau$} if
    \begin{itemize}
        \item[(i)] $F$ is an entire function in all of its variables, i.e., $F$ has an absolutely convergent  power series expansion
        $$F(\z)=\sum\limits_{\z\in\NN^Q}a_\k\z^\k, \qquad \z\in \CC^Q$$
        with constant coefficients $a_\k\in\CC$.
        \item[(ii)] For any $\epsilon>0$ there exists a positive number $A_\epsilon$ such that for all $\z\in\CC^Q$, the inequality
        $$|F(\z)|\leq A_\epsilon\exp\left((\tau+\epsilon)\sum\limits_{j=1}^Q|z_j|\right)$$
        is satisfied.
    \end{itemize}
    {\rm (b)} If $\mathbf{v}=(v_1,\dots,v_Q)\in\RR_+^Q$, then $F$ is said to be an \textbf{entire function of exponential type $\v$} if the function $\z\mapsto F(z_1/v_1,\cdots,z_Q/v_Q)$ is an entire function of finite exponential type $1$. \\
 {\rm (c)} Let $\mathbf{v}\in \ell^1(\CC)$. A function $F:c_0(\CC)\to\CC$ is called an \textbf{entire function of finite exponential type $\mathbf{v}$} if, for every $Q\in \NN$, the function $(z_1,\cdots,z_Q)\to F(z_1,\cdots, z_Q,0,0,\cdots)$ is an entire function of finite exponential type $(v_1,\cdots,v_Q)$. 

\end{definition}

An important example of entire functions of finite exponential type on $c_0(\CC)$ is the mapping
$$
\z\in c_0(\CC)\mapsto \int \exp(-i\z\cdot\x)d\mu(\x),
$$
where $\mu$ is a probability measure supported on the infinite cube $[-1,1]^\infty$.

\subsection{Chebyshev polynomials}\label{sec:chebyshev}

Let $d\in \NN$. 
\begin{equation}\label{eq:chebwt}
v_d(\x)=\pi^{-d}\prod_{k=1}^d(1-x_k^2)^{-1/2}, \qquad \x=(x_1,\cdots, x_d)\in I^d.
\end{equation}
The space $L^p(I^d)$ will refer to the space of all $f$ for which
\be\label{eq:lpd_def}
\|f\|_{d,p}=\begin{cases}
\disp\left\{\int_{I^d} |f(\x)|^pv_d(\x)d\x\right)^{1/p}, &\mbox{ if $0<p<\infty$,}\\[1ex]
\disp\esssup_{\x\in I^d}|f(\x)|, &\mbox{if $p=\infty$,}
\end{cases}
\ee
is finite. 
As usual, we will identify two functions if they are equal almost everywhere.

We denote the space of all polynomials in $d$ variables of coordinatewise degree $<n$ by  $\Pi_n^d$.

Next, we define Chebyshev polynomials.
We define Chebyshev polynomials in the univariate case by first setting $x=\cos\theta$ for $x\in [-1,1]$ and define
\begin{equation}\label{eq:unicheb}
p_k(x)=\begin{cases}
1, &\mbox{ if $k=0$},\\
\sqrt{2}\cos (k\theta), &\mbox{if $k=1,2,\cdots$}.
\end{cases}
\end{equation}
We note that the expression $p_k$ is a polynomial of degree $k$ in $x$, and the normalization is set so that
\begin{equation}\label{eq:uniortho}
\int_{-1}^1 p_k(x)p_j(x)v_1(x)dx=\delta_{k,j}.
\end{equation}
The multivariate Chebyshev polynomials are defined by
\begin{equation}\label{eq:multicheb}
p_\k(\x)=\prod_{j=1}^d p_{k_j}(x_j), \qquad \k=(k_1,\cdots, k_d)\in \ZZ_+^d, \ \x=(x_1,\cdots,x_d)\in I^d,
\end{equation}
and satisfy
\begin{equation}\label{eq:multuortho}
\int_{I^d}p_\k(\x)p_\j(\x)v_d(\x)d\x=\delta_{\k,\j}.
\end{equation}
We note that even though we have defined the Chebyshev polynomials by their values on $I^d$, they are actually defined on $\CC^d$ because they are polynomials.

Any function $f\in L^2(I^d)$ admits an formal expansion
\begin{equation}\label{eq:chebexpansion}
f=\sum_{\k\in\NN^d}\hat{f}(\k)p_\k,
\end{equation}
where the \textbf{Chebyshev coefficients} are defined by
\begin{equation}\label{eq:fourcoeff}
\hat{f}(\k)=\int_{I^d}f(\y)p_\k(\y)v_d(\y)d\y, \qquad \k\in \ZZ_+^d.
\end{equation}

For $f\in L^1(I^d)$, we define the partial sums of \eqref{eq:chebexpansion} by
\begin{equation}\label{eq:partialsum}
s_n(f)(\x)=\sum_{|\k|_1< n}\hat{f}(\k)p_\k(\x), \qquad S_n(f)(\x)=\sum_{|\k|_1=n}\hat{f}(\k)p_\k(\x), \qquad n\in\NN.
\end{equation}

There is an important formula that relates Chebyshev expansions with Laurent expansions of meromorphic functions. 
We note that for $0<\rho<1$, the two branches of the Joukowski transform $w=z+\sqrt{z^2-1}$ map $U_\rho$ to the annulus $\rho<|w|<1/\rho$. Hence, for a function $f$ analytic on $U_\rho$ for some $\rho>0$, the function $g(w)=f((w+w^{-1})/2$ is analytic on the annulus. 
The Laurent expansion of $g$ is given by
\be\label{eq:chebtolaurent}
\hat{f}(0)+(1/2)\sum_{k=1}^\infty \hat{f}(k)(w^k+w^{-k}).
\ee
Thus, the coefficients, the partial sums, and the remainder $f-s_n(f)$ can be expressed as contour integral over appropriate circles in the $w$ plane.
For multivariate functions, of course, one uses tensor products of circles.

If $\bs r>0$, we may define Chebyshev polynomials on $I_{\bs r}$ by 
\be\label{eq:anysochebdef}
p_{\k,\bs r}(\x)=\prod_{j=1}^d p_{k_j}(\x/\bs r),
\ee
 and the corresponding weights by 
\be\label{eq:anisochebwt}
v_{\bs r}(\x)=\pi^{-d}\left(\prod_{j=1}^d (r_j^2-x_j^2)\right)^{-1/2}.
\ee
Of course, one has the orthogonality relation
\be\label{eq:anysocheborth}
\int_{I_{\bs r}} p_{\k,\bs r}(\x)p_{\m,\bs r}(\x)v_{\bs r}(\x)d\x=\delta_{\k,\m}.
\ee
The Chebyshev coefficients and partial sums are defined in an obvious way and will be indicated by an extra subscript $\bs r$; e.g., $s_{n,\bs r}$.



\section{Main results}\label{bhag:main}

In this section, we define compact spaces of analytic and entire functions and state our theorems about their entropies. 
In Section~\ref{sec:directsumprod}, we encapsulate the procedure in some abstraction.
The spaces for analytic functions and their entropy estimates are given in Section~\ref{sec:analytic}.
Analogous results for entire functions are given in  Section~\ref{sec:entire}.
We conclude with estimates on the entropy of functionals in Section~\ref{sec:functional}.

\subsection{Direct sums and products}\label{sec:directsumprod} 

Let $X$ be a Banach space. We assume that there exists a sequence of finite dimensional subspaces $X_j$, $j=0,1,\dots$, $b_j=\mathsf{dim}(X_j)$, $Y_k=\bigoplus\limits_{j=0}^{k-1}X_j$, $d_k=\mathsf{dim}(Y_k)=\sum\limits_{j=0}^{k-1} b_j$, such that $\bigcup\limits_{k=0}^\infty Y_k$ is dense in $X$.
In particular, we assume that for any $f\in X$, there is a unique sequence $\{f_j\in X_j\}_{j=0}^\infty$ such that we have a formal expansion of the form $f\sim \sum_j f_j$. (An example is the space $C(I^q)$, $X_j=\Pi_j^q$, and $f_j=S_{j-1}(f)$, as in \eqref{eq:partialsum}.)
We write $\mathsf{Proj}_j(f)=f_j$, and assume that $\mathsf{Proj}_j$ is a continuous operator for each $j$. Generalizing the notation established in Section~\ref{sec:chebyshev}, we define
$$
s_k(f)=\sum_{j=0}^{n-1} \mathsf{Proj}_j(f),\quad S_k(f)=\mathsf{Proj}_k(f).
$$
Let $\mathfrak{K}$ be a compact subset of $X$. Then
$$
\lim_{n\to\infty}\sup_{f\in \mathfrak{K}}\mathsf{dist}(f,Y_n)=0. 
$$
In this paper, we are interested in $\mathfrak{K}$ such that
\be\label{eq:abstract_four_conv}
\lim_{n\to 0}\sup_{f\in \mathfrak{K}}\|f-s_n(f)\|=0.
\ee
More precisely, with a summable sequence $\{\Delta_j\}_{j=0}^\infty$ of positive numbers, we define
\be\label{eq:compact_set_def}
\mathfrak{K}=\left\{f\in X:\ \|S_j(f)\|\leq\Delta_j,\ j\in\NN\right\}, \qquad \mathfrak{K}_j=\mathsf{Proj}_j(\mathfrak{K}), \qquad \widetilde{\mathfrak{K}}_n=\bigoplus_{j=0}^{n-1}\mathfrak{K}_j.
\ee
Let $\epsilon>0$. 
In order to estimate the entropy of $\mathfrak{K}$, we observe first that
in view of \eqref{eq:abstract_four_conv},
 there exists some $n\in\NN$ such that 
 $$
 \sup_{f\in \mathfrak{K}}\|f-s_n(f)\|\leq\epsilon/2.
 $$ 
Thus, any $\epsilon/2$-net of the set $\widetilde{\mathfrak{K}}_n$ is an $\epsilon$-net of $\mathfrak{K}$, and any $\epsilon$-net of the set $\mathfrak{K}$ is an $\epsilon$-net of $\widetilde{\mathfrak{K}}_n$. Thus,
\be\label{eq:abstract_reduction}
H_\epsilon\left(\widetilde{\mathfrak{K}}_n,X\right)\leq H_\epsilon\left(\mathfrak{K},X\right)\leq H_{\epsilon/2}\left(\widetilde{\mathfrak{K}}_n,X\right).
\ee
Therefore, in order to estimate the entropy of $\mathfrak{K}$, we only need to esitimate the entropy of $\widetilde{\mathfrak{K}}_n$.

For this purpose, it is convienient  to identify $\widetilde{\mathfrak{K}}_n$ with a tensor product of balls.

We consider the space $X_{\Pi,n}=\prod\limits_{j=0}^{n-1}X_j$, and the mapping $\mathcal{T}_n(f)=(\mathsf{Proj}_0(f), \cdots, \mathsf{Proj}_{n-1}(f))$ from $\bigoplus\limits_{j=0}^{n-1}X_j$ to $X_{\Pi,n}$.
Obviously, $\mathcal{T}_n$ is a one-to-one mapping.
If    $1\le p\le \infty$, we may define a norm on $X_{\Pi,n}$ by
\be\label{eq:productnorm}
\|\mathcal{T}_n(f)\|_{\Pi,p,n}=\left| (\|\mathsf{Proj}_0(f)\|, \dots,\|\mathsf{Proj}_{n-1}(f)\|)\right|_p.
\ee
Since all the spaces involved are finite dimensional, there exist positive constants $A_{n,p}, B_{n,p}$ such that
\be\label{eq:abstractnikolskii}
A_{n,p}\|f\|\le \|\mathcal{T}_n(f)\|_{\Pi,p,n}\le B_{n,p}\|f\|,\quad f\in Y_n.
\ee
Next, we note that $\mathfrak{K}_j$ is a ball in the finite dimensional space $X_j$:
\be\label{eq:Kdegapprox}
\mathfrak{K}_j=\left\{f\in X_j: \|f\|\le \Delta_j\right \},\quad j\in\NN.
\ee
So, we can view $\widetilde{\mathfrak{K}}_n$ via the mapping $\mathcal{T}_n$ as a product of the balls $\mathfrak{K}_j$. 
The entropy of such a product is given in  \cite[Proposition~1.3]{lorentz_advanced}.
To summarize, the entropy of $\mathfrak{K}$ can be estimated as in the following theorem.

\begin{theorem}\label{theo:abstractentropy}
Let $\epsilon>0$, $1\le p, r\le \infty$, and  we recall the notation established in \eqref{eq:abstractnikolskii}, \eqref{eq:Kdegapprox}. We have
\begin{equation}\label{eq:abstractentropy}
\begin{split}
\sum_{j=0}^{N-1} b_j\log\left(\frac{\Delta_j}{2B_{N,r}\epsilon}\right)\leq H_{\epsilon}(\mathfrak{K})\le \sum_{j=0}^{M-1} b_j\log\left(\max\left(\frac{6M^{1/p}\Delta_j}{ A_{M,p}\epsilon},1\right)\right),
\end{split}
\end{equation}
which holds for all $N\geq1$ and $M\geq\mathcal N(\epsilon/2):=\min\left\{m\in\NN:\ \sum\limits_{n=m}^\infty\Delta_n\leq\epsilon/2\right\}$.
\end{theorem}

\subsection{Spaces of analytic functions}\label{sec:analytic}
Let $q\in\NN$, $\rho\in(0,1)$. 
In view of Theorem \ref{theo:identifyanal}, we define the class of analytic functions by
\begin{equation}\label{defA}
\mathcal A_\rho=\left\{f: I^q\to\RR:\ \left\|S_{n}(f)\right\|_{L^2( I^q)}\leq\rho^{n},\quad n\in\NN\right\},
\end{equation}
The goal of this section is to prove Theorem~\ref{entropy analy} to estimate the entropy of $\mathcal{A}_\rho$.

We will use Theorem \ref{theo:abstractentropy} with $A_{n,2}=B_{n,2}=1$, $\Delta_j=\rho^j$, $b_j=\binom{j+q-1}{q-1}$, $j=0,1,\dots$, to obtain the following theorem.

\begin{theorem}\label{entropy analy}

\begin{itemize}
\item [(a)] For $\epsilon<\frac{2}{\sqrt{1-\rho^2}}\left(\frac{9}{2}(\rho^{-2}-1)(q+1)\right)^{(q+1)\log\frac{1}{\rho}}$,
\begin{equation}\label{entropyA2}
H_\epsilon(\mathcal{A},\|\cdot\|_{L_2( I^q)})\leq\frac{4e^{q+1}}{\sqrt{2\pi}}\left(1+\frac{\log\left(\frac{2\rho}{\sqrt{1-\rho^2}}\frac{1}{\epsilon}\right)}{(q+1)\log\frac{1}{\rho}}\right)^{q+1}.
\end{equation}
\item [(b)] For $\epsilon<1/2$,
\begin{equation}\label{entropyA2lower}
H_\epsilon(\mathcal{A},\|\cdot\|_{L_2( I^q)})\geq\frac{2^{q+1}}{8\sqrt{2\pi(q+1)}}\left(1+\frac{\log\left(\frac{\rho^2}{4}\frac{1}{\epsilon}\right)}{(q+1)\log\frac{1}{\rho}}\right)^{q+1}.
\end{equation}
\item [(c)] Furthermore, for $\epsilon$ sufficiently small, the entropy of $\mathcal A_\rho$ satisfies
\begin{equation}\label{entropyA}
1-\frac{2(q+1)\log\frac{2}{\rho}}{\log\frac{1}{\epsilon}}\leq\frac{H_{\epsilon}(\mathcal A_\rho,\|\cdot\|_{L_2( I^q)})}{\frac{\log\frac{1}{\rho}}{(q+1)!}\left(\frac{\log\frac{1}{\epsilon}}{\log\frac{1}{\rho}}\right)^{q+1}}\leq1+\frac{2(q+1)\log\frac{1}{\rho}}{\log\frac{1}{\epsilon}}\left(\log\log\frac{1}{\epsilon}+\log\frac{18\sqrt{1-\rho^2}}{\rho^q}\right).
\end{equation}

\end{itemize}

\end{theorem}

\subsection{Spaces of entire functions}\label{sec:entire}

In this section, we are interested in the class of entire functions of finite exponential type, defined in \eqref{def:entire} below. We will use Theorem \ref{theo:abstractentropy} again to estimate the entropy of this class. 
The main difficulty in this section is keeping track of the dependence of the dimension $Q$. This is important when we consider functional classes in Section \ref{sec:functional}.

Let $Q\in\NN$, $\tau \in\left[1,\frac{Q}{2e^{3/2}\pi}\right]$, $\r\in\RR_+^Q$, $C=\left(\frac{2\pi}{Q}\right)^{Q/2}$, and let
$$\Lambda(N)=CN^{Q/2}\frac{\tau ^N}{N!},\quad N\in\NN.$$

Let $I_\r=\prod\limits_{j=1}^Q[-r_j,r_j]$ be a subset of $\RR^Q$. In view of Theorem \ref{thm entire} and \ref{converse}, we can define the class of entire functions by
\begin{equation}\label{def:entire}
\mathcal B_Q=\mathcal B_Q(\r,\tau )=\left\{F:I_\r\to\RR:\ \left\|S_N(F)\right\|_{L_\infty( I_\r)}\leq\Lambda(N),\  N\in\NN\right\}.
\end{equation}

\begin{theorem}\label{entropy entire}
\begin{itemize}
\item [(a)] Under the condition that
\be\label{eq:epsupcond}
\epsilon\leq\left(\frac{2\pi e\tau }{Q}\right)^{Q/2}\frac{4}{(e\tau )^{1/2}\exp(e^2\tau )},
\ee
the entropy of $\mathcal B_Q$ defined in \eqref{def:entire} satisfies
\begin{equation}\label{entropyB2}
\begin{split}
&H_\epsilon(\mathcal B_Q,\|\cdot\|_{L_\infty( I_\r)})\\
\leq&\frac{2}{3\sqrt{2\pi}}\left(\frac{2e}{Q}\right)^Q\left(\frac{\log\frac{4}{\epsilon}+\frac{Q}{2}\log\frac{2e\pi \tau }{Q}}{\log\left(\log\frac{4}{\epsilon}+\frac{Q}{2}\log\frac{2e\pi \tau }{Q}\right)-\log(e\tau )}+\frac{3Q}{4}\right)^{Q+1}\left(5\log\log\frac{1}{\epsilon}+\log\left((Q+1)^2(e\tau)^6\right)\right).
\end{split}
\end{equation}
\item [(b)] Let
$$\xi_\tau=\frac{16\max\{3e^2\tau,128\}\log\left(\max\{3e^2\tau,128\}\right)}{e}+2,$$
under the condition that
\be\label{eq:epslowcond}\epsilon\leq\left(\frac{2\pi}{Q}\right)^{Q/2}\frac{1}{4\sqrt{2\pi e\tau}}\xi_\tau^{-2\xi_\tau},\ee
\begin{equation}\label{entropyB2lower}
\begin{split}
H_\epsilon(\mathcal B_Q,\|\cdot\|_{L_\infty( I_\r)})\geq&\frac{1}{16\sqrt{\pi Q}}\left(\frac{1}{Q}\right)^Q\left(\frac{\log\frac{1 }{4\sqrt{2\pi e\tau}\epsilon}+\frac{Q}{2}\log\left(\frac{2\pi}{Q}\right)}{\log\left(\log\frac{1 }{4\sqrt{2\pi e\tau}\epsilon}+\frac{Q}{2}\log\left(\frac{2\pi}{Q}\right)\right)-\log(e\tau )}-\frac{5}{2}+Q\right)^Q\\
&\times\left(\frac{\log\frac{1 }{4\sqrt{2\pi e\tau}\epsilon}+\frac{Q}{2}\log\left(\frac{2\pi}{Q}\right)}{\log\left(\log\frac{1 }{4\sqrt{2\pi e\tau}\epsilon}+\frac{Q}{2}\log\left(\frac{2\pi}{Q}\right)\right)-\log(e\tau )}-\frac{3}{2}\right).
\end{split}
\end{equation}

\item [(c)] The following asymptotic result holds:
\begin{equation}\label{entropyB1}
\frac{1}{2Q!}\frac{\left(\log\frac{1}{\epsilon}\right)^{Q+1}}{\left(2\log\log\frac{1}{\epsilon}\right)^{Q}}(1+o(1))\leq H_\epsilon(\mathcal B_Q,\|\cdot\|_{L_\infty( I_\r)})\leq\frac{1}{Q!}\frac{\left(2\log\frac{1}{\epsilon}\right)^{Q+1}}{\left(\log\log\frac{1}{\epsilon}\right)^{Q}}(1+o(1))
\end{equation}
as $\epsilon\to0$, where the $o(1)$ term is $\sim Q\log\log(1/\epsilon)/\log(1/\epsilon)$.

\end{itemize}
\end{theorem}

\subsection{Space of functionals}\label{sec:functional}
In this section, we are interested in estimating the entropy of a class of functionals $\mathcal F$ on $\mathcal A_\rho$ with respect to the $\sup$-norm. 
Any functional in $\mathcal{F}$ can be viewed as a functional on the sequence of Chebyshev coefficients of the input function. 
We will define $\mathcal{F}$ to be a set of functionals that are entire functions of certain type as in Definition~\ref{def:entirefn}(c).

Under the assumption that the functionals are Lipschitz continuous, i.e.,
$$\sup_{\tilde{F}\in \mathcal{F}}\sup\limits_{f_1\neq f_2}\frac{|\tilde F(f_1)-\tilde F(f_2)|}{\|f_1-f_2\|_{L^2( I^q)}}\leq1,$$
we conclude for any $\epsilon>0$, there is some integer $n$ such that for any $\tilde{F}\in \mathcal{F}$,
$$\left|\tilde F(f)-\tilde F(s_{n+1}(f))\right|\leq\|f-s_{n+1}(f)\|_{L^2(I^q)}\leq\epsilon/2,\ F\in\mathcal{F}.$$
Consequently, the $\epsilon$-entropy of $\mathcal{F}$ is bounded by the $\epsilon/2$-entropy of
$$\left\{\tilde{F}\circ s_{n+1}:\ \tilde{F}\in\mathcal{F}\right\}.$$
In turn, for any $f\in \mathcal{A}_\rho$, $\tilde{F}\circ s_{n+1}(f)$ can be viewed as a function of the Chebyshev coefficients of $f$ up to order $n+1$.

We now define the set of functionals formally.
Let $q\in\NN$, $\rho<1$, $\mathcal A_\rho$ be as in Theorem \ref{entropy analy}. For any $n\in\NN$, write $Q=Q(n)=\binom{n+q}{q}$. The distance $\|f-s_{n+1}(f)\|_{L^2(I^q)}$ is bounded as
$$\|f-s_{n+1}(f)\|_{L^2(I^q)}\leq\left(\sum\limits_{j=n+1}^\infty\rho^{2j}\right)^{1/2}\leq\frac{\rho^{n+1}}{\sqrt{1-\rho^2}}.$$
By definition, the Chebyshev coefficients of $f\in\mathcal{A}_\rho$ satisfy $|\hat f(\k)|\leq \rho^{|\k|_1}$. 
Let $\tilde{\mathbf{r}}=(r_j)_{j=1}^\infty$ be defined by
$$r_j=\rho^\ell,\quad\binom{q+\ell-1}{q}< j\leq\binom{q+\ell}{q}$$
and let $\r=(r_j)_{j=1}^Q$.
Then for  $f\in\mathcal{A}_\rho$,   $\left(\hat f(\k)\right)_{|\k|_1\leq n} \in I_\r\subset\RR^Q$. 
Consequently, the functionals on the polynomial space $\Pi_{n+1}^q$ are identified as functions on $I_\r$ as follows:
$$F\left((a_\k)_{|\k|_1\leq n}\right)\mapsto \tilde F\left(\sum\limits_{\k\in\NN^Q}a_\k p_{\k,\r}\right).$$
The functionals on $\mathcal A_\rho$ we are concerned with are the functionals that induce entire functions of some type $\v$ by this process.

Let $\tilde{\v}=(v_j)_{j=1}^\infty$ be a nonnegative sequence, for any $n\in\NN$, let $Q=\binom{n+q}{q}$, $\r=(r_j)_{j=1}^Q$; we denote the class of functionals $\mathcal F_{n,\tilde{\v}}$ on $\mathcal A_\rho$ by
\begin{equation}\label{def functional N}
\mathcal F_{n,\tilde{\v}}:=\left\{\tilde F:\ \mathcal A_\rho\to\RR:\ \exists F\in\tilde{\mathcal B}_{n,\tilde{\v}}\hbox{ such that }\tilde F(f)=F\left(\left(\hat f(\k)\right)_{|\k|_1\leq n}\right),\  f\in\mathcal A_\rho\right\},
\end{equation}
where each $\tilde{\mathcal B}_{n,\tilde{\v}}$ is denoted as
$$\tilde{\mathcal B}_{n,\tilde{\v}}=\left\{F:I_\r\to\RR:\ \|S_N(F)\|_{L_\infty(I_\r)}\leq \left(\frac{2\pi}{Q}\right)^{Q/2}N^{Q/2}\frac{1}{N!}\left(\sum\limits_{j=1}^{Q}v_jr_j\right)^N\right\}.$$

We denote the class of functionals on $\mathcal A_\rho$ in this section $\mathcal F_{\tilde{\v}}$ as
\begin{equation}
\mathcal F_{\tilde{\v}}=\left\{\tilde F:\ \mathcal A_\rho\to\RR:\ \sup\limits_{f_1\neq f_2}\frac{|\tilde F(f_1)-\tilde F(f_2)|}{\|f_1-f_2\|_{L^2( I^q)}}\leq1,\ \tilde F\circ s_{n+1}\in\mathcal F_{n,\tilde{\v}},\quad\forall n\in\NN\right\}
\end{equation}
and the metric on $\mathcal F_{\tilde{\v}}$ is
$$\|F\|_{\mathcal F_{\tilde{\v}}}=\sup\limits_{f\in\mathcal A_\rho}|F(f)|,\quad F\in\mathrm{span}\left(\mathcal F_{\tilde{\v}}\right).$$
We estimate the entropy of the class $\mathcal F_{\tilde{\v}}\times\mathcal A_\rho$,  with respect to the metric $\|\cdot\|$ defined by
$$\|\cdot\|=\|\cdot\|_{L^2( I^q)}+\|\cdot\|_{\mathcal F_{\tilde{\v}}}.$$
Our main theorem in this subsection gives a bound of the entropy $H_\epsilon\left(\mathcal F_{\tilde{\v}}\times\mathcal A_\rho,\|\cdot\|\right)$.



\begin{theorem}\label{total entropy}
Let $\tilde{\v}=(v_j)_{j=1}^\infty$ be denoted by
$$v_j=\frac{1}{2e^{3/2}\pi\rho^\ell},\quad\binom{q+\ell-1}{q}< j\leq\binom{q+\ell}{q},$$
then for
$$\epsilon<\min\left\{\frac{2}{\sqrt{1-\rho^2}}\left(\frac{9}{2}(\rho^{-2}-1)(q+1)\right)^{(q+1)\log\frac{1}{\rho}},\frac{4\rho^q}{\sqrt{1-\rho^2}}\right\},$$
the entropy of $\mathcal F_{\tilde{\v}}\times\mathcal A_\rho$ is bounded by
\begin{equation}\label{upper total}
\begin{split}
&H_{\epsilon}(\mathcal F_{\tilde{\v}}\times\mathcal A_\rho,\|\cdot\|)\leq\frac{26q}{3\sqrt{2\pi}}\exp\left\{\gamma^q\log\left(\frac{e^{3/2}}{\pi}+2e\right)\right\}\gamma^{q}\log\gamma+\gamma^{q+1}
\end{split}
\end{equation}
with
$$\gamma=\frac{2e\log\frac{1}{\epsilon}}{q\log\frac{1}{\rho}}.$$
\end{theorem}
\begin{rem}\label{rem:lowbound}
{\rm At a first glance, the lower bound for the entropy of the set $\mathcal F_{n,\tilde{\v}}$ can be derived immediately from Theorem \ref{entropy entire}. 
However, in the definition of $\mathcal F_{\tilde{\v}}$, there is a Lipschitz condition under which we can only consider subsets of $\mathcal F_{n,\tilde{\v}}$ in the proof. 
These subsets do not fit our abstract framework. Therefore, we are not able to obtain a lower bound at this time. \qed
}
\end{rem}

\section{Computational issues}\label{bhag:computeissue}

\subsection{Generating analytic and bandlimited functions}\label{bhag:analbandgen}

A simple way to generate functions that are analytic on the interior of the poly-ellipse $ U_\rho\subset \CC^d$:
\be\label{eq:ellipseeqn}
\x=\x(\bs\theta)=\frac{\rho+\rho^{-1}}{2}\cos(\bs\theta), \quad \y=\y(\bs\theta)=\frac{\rho-\rho^{-1}}{2}\sin(\bs\theta), \qquad \bs\theta\in (-\pi,\pi]^d,
\ee
is the following.
We take a random sample  $\{\bs\theta_j\}_{j=1}^M$ on $(-\pi,\pi]^d$ and generate points $\w_j=\x(\bs\theta_j)+i \y(\bs\theta_j)$ on $ U_\rho$. 
We also take a random sample $\{a_j\}_{j=1}^{M}$ from some compact subset of $\CC$.
Then the function 
$$
f(\z)=\sum_{j=1}^M \frac{a_j}{\w_j-\z}+ \sum_{j=1}^M \frac{\overline{a_j}}{\overline{\w_j}-\z}
$$
is clearly analytic in the interior of $ U_\rho$ and real-valued on $[-1,1]^d$. 
A probability density on $[-1,1]^d$ can be obtained by normalizing $f(\x)^2$ to have integral equal to $1$.
Different choices of the random samples yield different distributions.

To generate  band-limited functions on $\ell^2$, we use a similar idea. 
We consider random sequences $\w_j\in \ell^2$, and samples $\{a_j\}$ on a complex ellipsoid, $j=1,\cdots,M$. 
For any such sequence and random sample, we have a band-limited function of the form
$$
f(\z)=\sum_{j=1}^M a_j\exp(i\z\cdot \w_j)+\sum_{j=1}^M \overline{a_j}\exp(-i\z\cdot \overline{\w_j}),
$$
which are real-valued for real sequences $\z$.

\subsection{Generating $\epsilon$-nets on ellipsoids}\label{bhag:netgeneration}

We note first that for any norm $\|\cdot\|$ on $\RR^d$, the ellipsoid $\mathbb{B}(\x_0, {\bs r})$ centered at $\x_0$ is parametrized by
$$
x_j=x_{0,j}+r_jy_j,
$$
where $\y=(y_j)$ belongs to the unit ball $\mathbb{B}(\bs 0, 1)$.
Therefore, it is enough to generate a net for this ball; the net on the ellipsoid can be generated by appropriate scaling.
Accordingly, we describe the generation of an $\epsilon$-net for 
$\mathbb{B}(\bs 0, 1)$.

In \cite[Proof of Lemma~7.1]{mhaskar2020kernel}, we have proved that if $\delta\in (0,1)$,
$$
M\ge (4/\epsilon)^d\log\left(\frac{(12/\epsilon)^d}{\delta}\right),
$$
and $\C=\{\z_1,\cdots,\z_M\}$ is a random sample from the uniform distribution on $\BB(\bs 0, 1)$, then with probability exceeding $1-\delta$, $\C$ is an $\epsilon/2$-net for $\BB(\bs 0, 1)$. 
To find a minimal $\epsilon$-net, we use a greedy algorithm: start with $\C=\{z_1\}$, and for $j=2,\cdots,M$, add the point $\z_j$ to $\C$ if $\mathsf{dist}(\C,\z_j)\ge \epsilon/2$.
Then clearly, $\C$ is an $\epsilon/2$-separated subset and $\epsilon$-net of $\BB(\bs 0, 1)$.

\section{Proofs}\label{bhag:proofs}

This section is organized as follows. In Section \ref{subse:combi}, we introduce some basic lemmas on binomial coefficients, which are used multiple times in the rest of the proof. In Section \ref{pf:abstractentropy}, we prove Theorem \ref{theo:abstractentropy}. This theorem is an abstract theorem, which can be applied to prove the entropy of analytic and entire function classes. Section \ref{pf:analy entropy}
 is the proof of Theorem \ref{entropy analy}. Section \ref{pf:entire upper} and Section \ref{pf:entire lower} are the proof of Theorem \ref{entropy entire}. Section \ref{pf:functional} is the proof of Theorem \ref{total entropy}, which shows the entropy of functional classes defined in Section \ref{sec:functional}.

\subsection{Combinatorial identities and inequalities}\label{subse:combi}
\begin{lemma}\label{lemma:combinatorics}
Let $n, d\in \NN$, $n\ge 1$. 
Then we have
\be\label{eq:combconv}
\sum_{j=0}^{n}\binom{j+d-1}{d-1}=\binom{n+d}{d},
\ee
\be\label{eq:combave}
\sum_{j=0}^{n} j\binom{j+d-1}{d-1}=d\binom{n+d}{d+1}, \qquad \sum_{j=1}^{n}(n- j)\binom{j+d-1}{d-1}=\binom{n+d}{d+1}.
\ee

\end{lemma}

\begin{proof}[Proof of Lemma~\ref{lemma:combinatorics}]
\eqref{eq:combconv} follows by noticing that $\binom{d-1}{d-1}=\binom{d}{d}=1$ and
\begin{equation*}
\begin{split}
\binom{j-1+d}{d}+\binom{j+d-1}{d-1}=&\frac{(n+d-1)\dots n}{d!}+\frac{(n+d-1)\dots(n+1)}{(d-1)!}\\
=&\frac{(n+d-1)\dots (n+1)}{(d-1)!}\left(1+\frac{n}{d}\right)=\frac{(n+d)\dots (n+1)}{d!}\\
=&\binom{j+d}{d}.
\end{split}
\end{equation*}

The first identity in \eqref{eq:combave} is given by
\begin{equation}\label{comb identity}
\begin{split}
&\sum\limits_{j=0}^{n}j\binom{j+d-1}{d-1}=\sum\limits_{j=0}^{n}j\frac{(j+d-1)\dots(j+1)}{(d-1)!}=\sum\limits_{j=0}^{n}\frac{(j+d-1)\dots(j+1)j}{d!}d\\
=&\sum\limits_{j=0}^{n}d\binom{j+d-1}{d}=d\binom{n+d}{d+1}.
\end{split}
\end{equation}

The second identity in \eqref{eq:combave} is a simple calculation using the first identity and \eqref{eq:combconv}.

\end{proof}

\begin{lemma}\label{lemma:asym comb}
Let $n,d\in\NN$, then
\begin{equation}\label{bound bin}
    \frac{1}{8\sqrt{2\pi(d+1)}}\left(\frac{2(n+d)}{d+1}\right)^{d+1}\leq\binom{n+d}{d+1}\leq\frac{2}{\sqrt{2\pi}}\left(\frac{e(n+d)}{d+1}\right)^{d+1}.
\end{equation}
In particular, if $n$ has the form
\be\label{form N gen}
n_1=\left\lfloor a(h+b)\right\rfloor,\quad n_2=\left\lfloor a(h-b)\right\rfloor
\ee
for some constants $a>0$, $b\geq0$ and $h\geq\left(b+\frac{d}{a}\right)(d+1)$, then
\be\label{eq:bin asym}
\begin{split}
\binom{n_1+d}{d+1}\leq&\frac{(ah)^{d+1}}{(d+1)!}\left(1+\frac{2(d+1)(ab+d)}{ah}\right)\\
\binom{n_2+d}{d+1}\geq&\frac{(ah)^{d+1}}{(d+1)!}\left(1-\frac{2(d+1)(ab+1)}{ah}\right).
\end{split}
\ee

\end{lemma}

\begin{proof}[Proof of Lemma~\ref{lemma:combinatorics}]
By Stirling's approximation formula,
$$\sqrt{2\pi k}\left(\frac{k}{e}\right)^k\leq k!\leq2\sqrt{2\pi k}\left(\frac{k}{e}\right)^k.$$
This gives
\begin{equation}\label{binom asy}
\begin{split}
&\frac{1}{4}\left(1+\frac{d+1}{n-1}\right)^{n-1}\left(\frac{n-1}{d+1}+1\right)^{d+1}\sqrt{\frac{n+d}{2\pi(d+1)(n-1)}}\leq\binom{n+d}{d+1}\\
\leq&2\left(1+\frac{d+1}{n-1}\right)^{n-1}\left(\frac{n-1}{d+1}+1\right)^{d+1}\sqrt{\frac{n+d}{2\pi(d+1)(n-1)}}
\end{split}
\end{equation}
and hence
\begin{equation*}
    \frac{1}{8\sqrt{2\pi(d+1)}}\left(\frac{2(n+d)}{d+1}\right)^{d+1}\leq\binom{n+d}{d+1}\leq\frac{2}{\sqrt{2\pi}}\left(\frac{e(n+d)}{d+1}\right)^{d+1}.
\end{equation*}

Suppose now
$$n_1=\left\lfloor a(h+b)\right\rfloor,$$
then
\begin{equation*}
\begin{split}
\binom{n_1+d}{d+1}\leq&\frac{(ah+ab+d)\dots(ah+ab)}{(d+1)!}\leq\frac{(ah)^{d+1}}{(d+1)!}\left(1+\frac{ab+d}{ah}\right)^{d+1}.
\end{split}
\end{equation*}
Since $h\geq\left(b+\frac{d}{a}\right)(d+1)$, we have
$$\frac{ab+d}{ah}\leq\frac{1}{d+1},$$
then
$$\binom{n_1+d}{d+1}\leq\frac{(ah)^{d+1}}{(d+1)!}\left(1+\frac{2(d+1)(ab+d)}{ah}\right).$$
This proves the upper bound in \eqref{eq:bin asym}.

On the other hand, $h\geq\left(b+\frac{1}{a}\right)(d+1)$, hence
\begin{equation*}
\begin{split}
\binom{n_2+d}{d+1}\geq&\frac{(ah-ab+d-1)\dots(ah-ab-1)}{(d+1)!}\geq\frac{(ah)^{d+1}}{(d+1)!}\left(1-\frac{ab+1}{ah}\right)^{d+1}.
\end{split}
\end{equation*}
Similarly, for $h\geq\left(b+\frac{d}{a}\right)(d+1)$, we have
$$\binom{n_2+d}{d+1}\geq\frac{(ah)^{d+1}}{(d+1)!}\left(1-\frac{2(d+1)(ab+1)}.{ah}\right)$$
This proves the lower bound in \eqref{eq:bin asym}.

\end{proof}

\begin{rem}\label{rem:logrmk}
{\rm
We will use the following estimate without explicit reference many times in the  proofs below.
\be\label{eq:logineq}
x^\alpha-\log x\ge (1/\alpha)\log(e\alpha), \qquad x, \alpha>0. 
\ee 
This can be checked easily by computing the minimum of the function $y\mapsto e^{\alpha y}-y$, $y\in\RR$.
}
\end{rem}

\subsection{Proof of Theorem \ref{theo:abstractentropy}}\label{pf:abstractentropy}

The proof of Theorem~\ref{theo:abstractentropy} requires the following lemma \cite[Proposition~1.3]{lorentz_advanced}.
\begin{lemma}\label{lemma:ballentropy}
Let $d\in\NN$, $(Y,\|\cdot\|_Y)$ be a $d$-dimensional normed linear space,  and $B_r=\{x\in Y : \|x\|\le r\}$. Then
\be\label{eq:ballentropy}
d\log (r/(2\epsilon)) \le C_{2\epsilon}(B_r,\|\cdot\|_Y)\le H_\epsilon(B,\|\cdot\|_Y)\le d\log\left( \max(3r/\epsilon,1)\right).
\ee
\end{lemma}

\begin{proof}[Proof of Theorem~\ref{theo:abstractentropy}]
In this proof, observe
$$\mathfrak{K}_j=\left\{g\in X_j:\ \|g\|\leq\Delta_j\right\}.$$
Let   $\C_j$ be an $\eta_1=\disp\frac{A_{M,p}\epsilon}{2M^{1/p}}$-net for each $\mathfrak{K}_j$, $j=0,\cdots, M-1$. 
Then it is easily verified that $\prod_{j=0}^{M-1}\C_j$ is an  $(A_{M,p}\epsilon/2)$-net for $\prod_{j=1}^{M-1}\mathfrak{K}_j $.
Therefore,  \eqref{eq:abstractnikolskii} shows that  $\mathcal{T}_M^{-1}(\prod_{j=0}^{M-1}\C_j)$ is an $\epsilon$-net for $\mathfrak{K}$ with respect to the norm of $X$. 
Since the cardinality of $\mathcal{T}_M^{-1}(\prod_{j=0}^{M-1}\C_j)$ is the same as that of $\prod_{j=0}^{M-1}\C_j$, it follows that
$$
H_\epsilon(\mathfrak{K}, X)\le \sum_{j=0}^M H_{\eta_1}(\mathfrak{K}_j ,X).
$$
Since each $\mathfrak{K}_j $ is a ball of radius $ \Delta_j$ in the $b_j$-dimensional space $X_j$, Lemma~\ref{lemma:ballentropy} leads to
$$
H_\epsilon(\mathfrak{K}, X)\le \sum_{j=0}^M b_j\log\left(\max\left( 3\Delta_j/\eta_1,1\right)\right).
$$
This proves the second inequality in \eqref{eq:abstractentropy}.

The proof of the first inequality in \eqref{eq:abstractentropy} is similar. We let $\eta_1= 2\epsilon /B_{N,r}$ and let $\tilde{\C}_j$  be the maximal $\eta_1$-separated subset of each $\mathfrak{K}_j$, $j=0,\cdots, N-1$. 
Then $\prod_{j=0}^{N-1} \tilde{\C}_j$ is an $\eta_1$-separated subset of $\widetilde{\mathfrak{K}}_N$, and hence, \eqref{eq:abstractnikolskii} shows that $\mathcal{T}^{-1}(\prod_{j=0}^{N-1} \tilde{\C}_j)$ an $2\epsilon=\eta_1B_{N,r}$-separated subset of $\mathfrak{K}$.
The cardinality of $\mathcal{T}^{-1}(\prod_{j=0}^{N-1} \tilde{\C}_j)$ is the same as that of $\prod_{j=0}^{N-1} \tilde{C}_j$.
Lemma~\ref{lemma:ballentropy} then shows that
$$
H_\epsilon(\mathfrak{K},X) \ge \sum_{j=0}^{N-1} C_{2\epsilon}(\mathfrak{K}_j, X)\ge \sum_{j=0}^{N-1} b_j\log\left(\frac{\Delta_j}{2B_{N,r}\epsilon}\right).
$$
This proves the first equation in \eqref{eq:abstractentropy}.
\end{proof}

\subsection{Proof of Theorem \ref{entropy analy}}\label{pf:analy entropy}

In this subsection, we apply Theorem \ref{theo:abstractentropy} with $p=r=2$ to give the proof of Theorem \ref{entropy analy}. In this case, $X=L^2(I^q)$, $X_j=\mathsf{span}\{p_\k:\ |\k|_1=j\}$, $\{\Delta_n\}_{n=0}^\infty=\{\rho^n\}_{n=0}^\infty$ and $\mathfrak{K}=\mathcal A_\rho$.

Moreover, for each $n\in\NN$ and each $f\in\bigoplus\limits_{j=0}^{n-1}X_j$,
\begin{eqnarray*}
\left\|\mathcal T_n(f)\right\|_{\Pi,2,n}&=&\left| \left(\left\|\mathsf{Proj}_0(f)\right\|_{L^2(I^q)}, \dots,\left\|\mathsf{Proj}_{n-1}(f)\right\|_{L^2(I^q)}\right)\right|_2\\
&=&\left(\sum\limits_{j=0}^{n-1}\left\|\sum\limits_{|\k|_1=j}\hat f(\k)p_\k\right\|_{L^2(I^q)}^2\right)^{1/2}\\
&=&\|f\|_{L^2(I^q)}.
\end{eqnarray*}
Hence, $A_{n,2}=B_{n,2}=1$.

\begin{proof}[Proof of Theorem \ref{entropy analy}]
In order to apply Theorem \ref{entropy analy}, we need to find an integer larger than $\mathcal N(\epsilon/2)$, which is the solution of the following inequalities:
$$\sum\limits_{n=N}^\infty\rho^{2n}=\frac{\rho^{2N}}{1-\rho^2}<\frac{\epsilon^2}{4}\leq\frac{\rho^{2N-2}}{1-\rho^2}=\sum\limits_{n=N-1}^\infty\rho^{2n}.$$
It gives
$$\mathcal N(\epsilon/2)=\left\lfloor\frac{\log\frac{2}{\epsilon}+\log\frac{1}{\sqrt{1-\rho^2}}}{\log\frac{1}{\rho}}\right\rfloor+1.$$


Now we estimate the bound in \eqref{theo:abstractentropy}. For simplicity, write 
\be\label{anal chose N1}
N_1=\left\lfloor\frac{\log\frac{2}{\epsilon}+\log\frac{1}{\sqrt{1-\rho^2}}}{\log\frac{1}{\rho}}\right\rfloor,
\ee
Then we can apply the second inequality in \eqref{theo:abstractentropy} to $M=N_1+1$. For each $n=0,\dots,N_1-1$,
\begin{eqnarray*}
\log\left(\frac{6\Delta_N\sqrt{N_1+1}}{\epsilon}\right)=\log\left(\frac{6\sqrt{N_1+1}\rho^n}{\epsilon}\right)\leq\log\left(\frac{6\sqrt{N_1+1}\rho^n}{2\frac{\rho^{N_1+1}}{\sqrt{1-\rho^2}}}\right)\leq\log\left(3\rho^{-1}\sqrt{(N_1+1)(1-\rho^2)}\rho^{n-N_1}\right),
\end{eqnarray*}
Hence, by Theorem \ref{theo:abstractentropy},
\begin{equation}\label{cover of Balls}
\begin{split}
H_{\epsilon}\left(\mathcal A_\rho,\|\cdot\|_{L_2( I^q)}\right)\leq\sum\limits_{n=0}^{N_1}\binom{n+q-1}{q-1}\left(\frac{1}{2}\log(N_1+1)+\log\left(3\rho^{-1}\sqrt{(1-\rho^2)}\right)+(N_1-n)\log\frac{1}{\rho}\right).
\end{split}
\end{equation}

By Lemma \ref{lemma:combinatorics},
\begin{equation}\label{entropyA N1}
\begin{split}
H_\epsilon(\mathcal A_\rho,\|\cdot\|_{L_2( I^q)})\leq&\binom{N_1+q}{q+1}\log\frac{1}{\rho}+\binom{N_1+q}{q}\frac{1}{2}\log\left(9(\rho^{-2}-1)(N_1+1)\right).
\end{split}
\end{equation}

Consider the upper bound in \eqref{entropyA N1}.
For $\epsilon<\frac{2}{\sqrt{1-\rho^2}}\left(\frac{9}{2}(\rho^{-2}-1)(q+1)\right)^{(q+1)\log\frac{1}{\rho}}$, we have
$$N_1+1\geq\frac{\log\frac{2}{\epsilon}+\log\frac{1}{\sqrt{1-\rho^2}}}{\log\frac{1}{\rho}}\geq(q+1)\log\frac{9(\rho^{-2}-1)(q+1)}{2}.$$
Take $\alpha=\log\frac{9(\rho^{-2}-1)(q+1)}{2}$; we have
\begin{displaymath}
\begin{split}
&\frac{1}{2}\log\left(9(\rho^{-2}-1)(q+1)\alpha\right)=\frac{1}{2}\left(\log(2\alpha)+\log\frac{9(\rho^{-2}-1)(q+1)}{2}\right)\\
=&\frac{1}{2}(\log(2\alpha)+\alpha)\leq\alpha.
\end{split}
\end{displaymath}
Since $N_1+1\geq(q+1)\alpha$,
$$\frac{1}{2}\log\left(9(N_1+1)(\rho^{-2}-1)\right)\leq\frac{N_1+1}{q+1}\leq\frac{N_1+q}{q+1}.$$
Hence, \eqref{entropyA N1} leads to
\begin{displaymath}
\begin{split}
H_\epsilon(\mathcal{A},\|\cdot\|_{L_2( I^q)})\leq&\frac{2}{\sqrt{2\pi}}\left(\frac{e(N_1+q)}{q+1}\right)^{q+1}\frac{1}{\rho}+\frac{2}{\sqrt{2\pi}}\left(\frac{e(N_1+q)}{q}\right)^{q}\times\frac{1}{2}\log\left(9(N_1+1)(1-\rho^2)\right)\\
\leq&\frac{2}{\sqrt{2\pi}}\left(\frac{e(N_1+q)}{q+1}\right)^{q+1}\frac{1}{\rho}+\frac{2}{\sqrt{2\pi}}\left(\frac{e(N_1+q)}{q+1}\right)^{q}\left(\frac{q+1}{q}\right)^{q}\frac{N_1+q}{q+1}\\
\leq&\frac{2}{\sqrt{2\pi}}\left(\frac{e(N_1+q)}{q+1}\right)^{q+1}\frac{1}{\rho}+\frac{2}{\sqrt{2\pi}}\left(\frac{e(N_1+q)}{q+1}\right)^{q}e\frac{N_1+q}{q+1}\frac{1}{\rho}\\
\leq&\frac{4e^{q+1}}{\sqrt{2\pi}}\left(\frac{N_1+q}{q+1}\right)^{q+1}\frac{1}{\rho}.
\end{split}
\end{displaymath}
Involving our choice of $N_1$ \eqref{anal chose N1} in this formula,
$$H_\epsilon(\mathcal{A},\|\cdot\|_{L_2( I^q)})\leq\frac{4e^{q+1}}{\sqrt{2\pi}}\left(1+\frac{\log\left(\frac{2\rho}{\sqrt{1-\rho^2}}\frac{1}{\epsilon}\right)}{(q+1)\log\frac{1}{\rho}}\right)^{q+1}$$

Now we prove the asymptotic relation. Applying \eqref{eq:bin asym} with $h\gets\log\frac{1}{\epsilon}$ to \eqref{entropyA N1}, and noticing that $\binom{N_1+q}{q}=\binom{N_1+q}{q+1}\frac{q+1}{N_1}$, then for $\epsilon$ sufficiently small,
\begin{equation*}
\begin{split}
&H_\epsilon(\mathcal A_\rho,\|\cdot\|_{L_2( I^q)})\leq\\
&\frac{\log\frac{1}{\rho}}{(q+1)!}\left(\frac{\log\frac{1}{\epsilon}}{\log\frac{1}{\rho}}\right)^{q+1}\left(1+\frac{2(q+1)\left(\log\frac{2}{\sqrt{1-\rho^2}}+q\log\frac{1}{\rho}\right)}{\log\frac{1}{\epsilon}}\right)\left(1+\frac{(q+1)\log\left(9(\rho^{-2}-1)(N_1+1)\right)}{2N_1}\right).
\end{split}
\end{equation*}
Bounding $1/N_1$ by $\left(\frac{\log\frac{1}{\epsilon}}{\log\frac{1}{\rho}}-1\right)^{-1}$ and $\log(N_1+1)$ by $\log\log\frac{1}{\epsilon}$, we get
\begin{equation*}
\begin{split}
&H_\epsilon(\mathcal A_\rho,\|\cdot\|_{L_2( I^q)})\\
\leq&\frac{\log\frac{1}{\rho}}{(q+1)!}\left(\frac{\log\frac{1}{\epsilon}}{\log\frac{1}{\rho}}\right)^{q+1}\left(1+\frac{2(q+1)\left(\log\frac{2}{\sqrt{1-\rho^2}}+q\log\frac{1}{\rho}\right)}{\log\frac{1}{\epsilon}}\right)\left(1+(q+1)\frac{\log\frac{1}{\rho}}{\log\frac{1}{\epsilon}}\left(\log\log\frac{1}{\epsilon}+\log(9(\rho^{-2}-1))\right)\right)\\
\leq&\frac{\log\frac{1}{\rho}}{(q+1)!}\left(\frac{\log\frac{1}{\epsilon}}{\log\frac{1}{\rho}}\right)^{q+1}\left(1+\frac{2(q+1)\left(\log\frac{2}{\sqrt{1-\rho^2}}+q\log\frac{1}{\rho}\right)}{\log\frac{1}{\epsilon}}+2(q+1)\frac{\log\frac{1}{\rho}}{\log\frac{1}{\epsilon}}\left(\log\log\frac{1}{\epsilon}+\log(9(\rho^{-2}-1))\right)\right)\\
\leq&\frac{\log\frac{1}{\rho}}{(q+1)!}\left(\frac{\log\frac{1}{\epsilon}}{\log\frac{1}{\rho}}\right)^{q+1}\left(1+\frac{2(q+1)\log\frac{1}{\rho}}{\log\frac{1}{\epsilon}}\left(\log\log\frac{1}{\epsilon}+\log\frac{18\sqrt{1-\rho^2}}{\rho^{q+1}}\right)\right).
\end{split}
\end{equation*}

Next we prove the lower bound; for this purpose, we chose $N_2$ to be as large as we can under the restriction $\log\left(\frac{\delta_{N_2}}{2\epsilon}\right)=\log\left(\frac{\rho^{N_2}}{2\epsilon}\right)\geq0$.

Solving the inequalities $\rho^{N_2+1}<2\epsilon\leq\rho^{N_2}$, we get
\be\label{anal chose N2}
N_2=\left\lfloor\frac{\log\frac{1}{2\epsilon}}{\log\frac{1}{\rho}}\right\rfloor-1.
\ee
Since $2\epsilon\leq\rho^{N_2}$, we have
$$\log\left(\frac{\rho^n}{2\epsilon}\right)\geq\log\left(\frac{\rho^n}{\rho^{N_2}}\right)\geq(N_2-n)\binom{n+q-1}{q-1}\log\frac{1}{\rho},\quad n\leq N_2.$$
Now by Theorem \ref{theo:abstractentropy} and \eqref{comb identity},
\begin{equation}\label{entropyA N2}
H_\epsilon(\mathcal A_\rho,\|\cdot\|_{L_2( I^q)})\geq\binom{N_2+q}{q+1}\log\frac{1}{\rho}.
\end{equation}

Similarly as before, involving $N_2=\left\lfloor\frac{\log\frac{1}{2\epsilon}}{\log\frac{1}{\rho}}\right\rfloor$ in \eqref{entropyA N2}, using \eqref{bound bin},
$$H_\epsilon(\mathcal{A},\|\cdot\|_{L_2( I^q)})\geq\frac{2^{q+1}}{8\sqrt{2\pi(q+1)}}\left(1+\frac{\log\frac{\rho^2}{2\epsilon}}{(q+1)\log\frac{1}{\rho}}\right)^{q+1}.$$

This completes the proof of \eqref{entropyA2}.

Consider the asymptotic relation. For $\epsilon$ sufficiently small,
\begin{displaymath}
\begin{split}
&H_\epsilon(\mathcal A_\rho,\|\cdot\|_{L_2( I^q)})\geq\binom{N_2+q}{q+1}\log\frac{1}{\rho}\geq\frac{\log\frac{1}{\rho}}{(q+1)!}\left(\frac{\log\frac{1}{\epsilon}}{\log\frac{1}{\rho}}\right)^{q+1}\left(1-\frac{2(q+1)\log\frac{2}{\rho}}{\log\frac{1}{\epsilon}}\right).
\end{split}
\end{displaymath}
The two inequalities prove \eqref{entropyA}.

\end{proof}

\subsection{Proof of Theorem \ref{entropy entire}: upper bound}\label{pf:entire upper}

As in the previous subsection, we apply Theorem \ref{theo:abstractentropy} to give the proof.
We recall the condition \eqref{eq:epsupcond} relating $Q$, $\tau$, and $\epsilon$.

\begin{proof}

For the upper bound, we apply Theorem \ref{theo:abstractentropy} with $p=1$, $X=L_\infty(I_\r)$, $X_j=\mathsf{span}\{p_{\k,\r}:\ |\k|_1=j\}$,
$$\Delta_N=\Lambda(N)=CN^{Q/2}\frac{\tau ^N}{N!},\quad N=0,1,\dots$$
and $\mathfrak{K}=\mathcal B_Q$.

In this case, for $F\in\bigoplus\limits_{j=0}^{n-1}X_j$,
$$\left\|\mathcal T_n(F)\right\|_{\Pi,1,n}=\sum\limits_{k=0}^{n-1}\left\|\mathsf{Proj}_k(F)\right\|_{L_\infty(I_\r)}\geq\|F\|_{L_\infty(I_\r)},$$
which means we can take $A_{n,1}=1$ for each $n\in\NN$. Now we only need to find a proper $M\geq\mathcal N(\epsilon/2)$ to apply Theorem \ref{theo:abstractentropy}.

To estimate the sum of $\Lambda(N)$, we investigate the decay of this sequence.
\begin{equation*}
\begin{split}
\frac{\Lambda(N+1)}{\Lambda(N)}=&\left(\frac{N+1}{N}\right)^{Q/2}(N+1)^{-1}\tau=\frac{\tau }{N+1}\left(\left(1+\frac{1}{N}\right)^N\right)^{\frac{Q}{2N}} \leq\frac{\tau }{N+1}\exp\left(\frac{Q}{2N}\right).
\end{split}
\end{equation*}
Then for $N\geq Q/\left(\log\left(\frac{Q}{2\tau }\right)\right)$,
\begin{equation*}
\begin{split}
\frac{\Lambda(N+1)}{\Lambda(N)}\leq\frac{\tau \log\left(\frac{Q}{2\tau }\right)}{Q}\exp\left(\frac{\log\left(\frac{Q}{2\tau }\right)}{2}\right)\leq\frac{\tau }{Q}\log\left(\frac{Q}{2\tau }\right)\sqrt{\frac{Q}{2\tau }}\leq\frac{\tau }{Q}\left(\sqrt{\frac{Q}{2\tau }}\right)^2=\frac{1}{2}
\end{split}
\end{equation*}
and consequently
\begin{equation}
\sum\limits_{N=M+1}^\infty\Lambda(N)\leq2\Lambda(M+1),\quad M\geq Q/\left(\log\left(\frac{Q}{2\tau }\right)\right).
\end{equation}
This enables us to find a proper $N_1\geq \mathcal N(\epsilon/2)$. To do this, we only need to find a proper integer $N_1$ with $2\Lambda(N_1+1)\leq\epsilon/2$. Applying Stirling's estimation, we have for any $M>\frac{Q-1}{2}$ and $M_0=M-\frac{Q-1}{2}$,
\begin{equation*}
\begin{split}
2\Lambda(M+1)\leq&2C(M+1)^{Q/2}\tau ^{M+1}\frac{1}{\sqrt{2\pi (M+1)}}\left(\frac{e}{M+1}\right)^{M+1}\\
\leq&2C(M+1)^{\frac{Q-1}{2}}\left(\frac{e\tau }{M+1}\right)^{M+1}\leq2C\left(\frac{e\tau }{M_0+1}\right)^{M_0+1}(e\tau )^{\frac{Q-1}{2}}.
\end{split}
\end{equation*}
So it suffices to find $M_0$ such that
$$2C\left(\frac{e\tau }{M_0+1}\right)^{M_0+1}(e\tau )^{\frac{Q-1}{2}}\leq\epsilon/2.$$
This inequality is equivalent to the inequality
$$\frac{M_0+1}{e\tau }\log\left(\frac{M_0+1}{e\tau }\right)\geq\frac{\log\frac{4}{\epsilon}+\log\left(C(e\tau )^{\frac{Q-1}{2}}\right)}{e\tau }.$$
Under the condition that
$$\epsilon\leq\left(\frac{2\pi e\tau }{Q}\right)^{Q/2}\frac{4}{(e\tau )^{1/2}\exp(e^2\tau )},$$
we have
$$\frac{M_0+1}{e\tau }\geq2\frac{\left(\log\frac{4}{\epsilon}+\frac{Q}{2}\log\frac{2e\pi \tau }{Q}\right)(e\tau )^{-1}}{\log\left(\log\frac{4}{\epsilon}+\frac{Q}{2}\log\frac{2e\pi \tau }{Q}\right)-\log(e\tau )}$$
for any
\begin{equation}\label{req M0}
M_0\geq2\frac{\log\frac{4}{\epsilon}+\frac{Q}{2}\log\frac{2e\pi \tau }{Q}}{\log\left(\log\frac{4}{\epsilon}+\frac{Q}{2}\log\frac{2e\pi \tau }{Q}\right)-\log(e\tau )}-1.
\end{equation}
Note $x\geq\frac{2y}{\log y}\Rightarrow x\log x\geq y$ for all $y>e$, we conclude
$$\frac{M_0+1}{e\tau }\log\left(\frac{M_0+1}{e\tau }\right)\geq\frac{\log\frac{4}{\epsilon}+\frac{Q}{2}\log\left(\frac{2e\pi\tau}{Q}\right)}{e\tau }$$
holds true for
$$\epsilon<\left(\frac{2\pi e\tau }{Q}\right)^{Q/2}\frac{4}{(e\tau )^{1/2}\exp(e^2\tau )}.$$
Then $2\Lambda\left(M_0+\frac{Q-1}{2}+1\right)\leq\epsilon/2$ for $M_0$ satisfying \eqref{req M0}.

Therefore, in order to make $\sum\limits_{N=N_1+1}^\infty\Lambda(N)\leq\epsilon/2$ hold true, it suffices to take
\be\label{def:N1 entire}
N_1=\left\lfloor2\frac{\log\frac{4}{\epsilon}+\frac{Q}{2}\log\frac{2e\pi \tau }{Q}}{\log\left(\log\frac{4}{\epsilon}+\frac{Q}{2}\log\frac{2e\pi \tau }{Q}\right)-\log(e\tau )}+\frac{Q-1}{2}\right\rfloor,
\ee
then $N_1+1$ is a proper integer for which Theorem \ref{theo:abstractentropy} can be applied. Now we use Stirling's approximation to bound $\Lambda(N)$ by
$$CN^{Q/2}\frac{\tau ^N}{N!}\leq C(e\tau )^N\left(\frac{1}{N}\right)^{N-Q/2}\leq C(e\tau )^N(N_1+1)^{Q/2},\ N\leq N_1+1$$
and notice that
\begin{equation*}
\begin{split}
\frac{6C(e\tau)^{N}(N_1+1)^{Q/2}(N_1+1)}{\epsilon}\geq&6(e\tau)^{N}\left(\frac{2\pi}{Q}\right)^{Q/2}(Q/2)^{Q/2+1}\left(\frac{Q}{2\pi e\tau}\right)^{Q/2}\frac{(e\tau)^{1/2}\exp(e^2\tau)}{4}\\
\geq&6(e\tau)^{N}\pi^{Q/2}e^{Q/4}\frac{(e\tau)^{1/2}\exp(e^2\tau)}{4}>1.
\end{split}
\end{equation*}
Therefore,
\begin{equation*}
\begin{split}
&\log\left(\max\left\{\frac{6\Lambda(N)(N_1+1)}{\epsilon},1\right\}\right)\leq\log\left(\frac{6C(e\tau )^N(N_1+1)^{Q/2+1}}{\epsilon}\right)\\
=&\left(\log\frac{6C}{\epsilon}+N\log(e\tau )+\frac{Q+2}{2}\log(N_1+1)\right)
\end{split}
\end{equation*}
and we can apply Lemma \ref{lemma:combinatorics} to get
\begin{equation}\label{entire upper bin1}
\begin{split}
H_\epsilon\left(\mathcal B_Q,\|\cdot\|_{L_\infty( I_\r)}\right)\leq&\sum\limits_{N=0}^{N_1}\binom{N+Q-1}{Q-1}\left(\log\frac{6C}{\epsilon}+N\log(e\tau )+\frac{Q+2}{2}\log(N_1+1)\right)\\
\leq&Q\binom{N_1+Q}{Q+1}\log(e\tau )+\binom{N_1+Q}{Q}\left(\log\frac{6C}{\epsilon}+\frac{Q+2}{2}\log(N_1+1)\right).
\end{split}
\end{equation}
Observing
$$\binom{N_1+Q}{Q+1}Q<N_1\binom{N_1+Q}{Q},$$
\begin{equation}\label{entire upper bin}
\begin{split}
H_\epsilon\left(\mathcal B_Q,\|\cdot\|_{L_\infty( I_\r)}\right)\leq\binom{N_1+Q}{Q}\left(N_1\log(e\tau )+\log\frac{6C}{\epsilon}+\frac{Q+2}{2}\log(N_1+1)\right).
\end{split}
\end{equation}

Next, we express the bound \eqref{entire upper bin} in terms of $\epsilon$.
We will apply Lemma \ref{lemma:asym comb}.

In this proof only, let 
$$B=\log\left(\frac{4}{\epsilon}\left(\frac{2e\pi \tau }{Q}\right)^{\frac{Q}{2}}\right),$$
then
$$N_1=\left\lfloor\frac{2B}{\log B-\log(e\tau)}+\frac{Q-1}{2}\right\rfloor\geq2e^2+\frac{Q-1}{2}.$$
Now we can apply \eqref{bound bin} to conclude
\begin{equation*}
\begin{split}
&H_\epsilon\left(\mathcal B_Q,\|\cdot\|_{L_\infty( I_\r)}\right)\leq\frac{2}{\sqrt{2\pi}}\left(\frac{e(N_1+Q)}{Q}\right)^Q\left(N_1\log(e\tau )+\log\frac{6C}{\epsilon}+\frac{Q+2}{2}\log(N_1+1)\right)\\
\leq&\frac{2}{\sqrt{2\pi}}\frac{e^Q2^Q}{Q^Q}\left(\frac{B}{\log B-\log(e\tau )}+\frac{3Q}{4}\right)^Q\\
&\times\left[\left(2\frac{B}{\log B-\log(e\tau )}+\frac{Q-1}{2}\right)\log(e\tau )+B+1-\frac{Q}{2}\log(e\tau)+\frac{Q+2}{2}\log\left(\frac{2B}{\log B-\log(e\tau )}+\frac{Q+1}{2}\right)\right].
\end{split}
\end{equation*}
Bounding
$$B+1-\frac{Q}{2}\log(e\tau)\leq\left(\frac{2B}{\log B-\log e\tau}+\frac{3Q}{2}\right)\frac{\log B-\log e\tau}{2}$$
and
\begin{equation*}
\begin{split}
&\frac{Q+2}{2}\log\left(\frac{2B}{\log B-\log(e\tau )}+\frac{Q+1}{2}\right)\leq\left(3+\frac{3Q}{2}\right)\times\frac{1}{3}\left(\log(2B)+\log\frac{Q+1}{2}\right)\\
\leq&\left(\frac{2B}{\log B-\log(e\tau )}+\frac{3Q}{2}\right)\times\frac{1}{3}\left(\log(2B)+\log\frac{Q+1}{2}\right),
\end{split}
\end{equation*}
we have
\begin{equation*}
\begin{split}
&H_\epsilon\left(\mathcal B_Q,\|\cdot\|_{L_\infty( I_\r)}\right)\leq\frac{2}{\sqrt{2\pi}}\left(\frac{2e}{Q}\right)^Q\left(\frac{B}{\log B-\log(e\tau )}+\frac{3Q}{4}\right)^Q\\
&\times\left(\frac{2B}{\log B-\log(e\tau )}+\frac{3Q}{2}\right)\left[\log(e\tau)+\frac{\log B-\log(e\tau)}{2}+\frac{1}{3}\left(\log2B+\log\frac{Q+1}{2}\right)\right]\\
\leq&\frac{2}{3\sqrt{2\pi}}\left(\frac{2e}{Q}\right)^Q\left(\frac{\log\frac{4}{\epsilon}+\frac{Q}{2}\log\frac{2e\pi \tau }{Q}}{\log\left(\log\frac{4}{\epsilon}+\frac{Q}{2}\log\frac{2e\pi \tau }{Q}\right)-\log(e\tau )}+\frac{3Q}{4}\right)^{Q+1}\left(5\log\log\frac{1}{\epsilon}+2\log(Q+1)+6\log(e\tau)\right).
\end{split}
\end{equation*}

Finally, consider the asymptotic relation of the bound when $\epsilon\to0$. A simple observation shows
$$N_1=2\frac{\log\frac{1}{\epsilon}}{\log\log\frac{1}{\epsilon}}(1+o(1))$$
as $\epsilon\to0$.

Then \eqref{entire upper bin} gives
\begin{equation*}
\begin{split}
H_\epsilon\left(\mathcal B_Q,\|\cdot\|_{L_\infty( I_\r)}\right)\leq&\frac{(N_1+Q)^Q}{Q!}\left(2\log\frac{1}{\epsilon}\right)(1+o(1))\\
\leq&\frac{1}{Q!}\left((1+o(1))\frac{2\log\frac{1}{\epsilon}}{\log\log\frac{1}{\epsilon}}\right)^Q\left(2\log\frac{1}{\epsilon}\right)(1+o(1))\\
\leq&\frac{1}{Q!}\frac{\left(2\log\frac{1}{\epsilon}\right)^{Q+1}}{\left(\log\log\frac{1}{\epsilon}\right)^{Q}}(1+o(1)).
\end{split}
\end{equation*}

\end{proof}

\subsection{Proof of Theorem \ref{entropy entire}: lower bound}\label{pf:entire lower}
In this section, we consider the lower bound. As in the last subsection, we apply Theorem \ref{theo:abstractentropy} with $p=1$, $X=L_\infty(I_\r)$, $X_j=\mathsf{span}\{p_{\k,\r}:\ |\k|_1=j\}$,
$$\Delta_N=\Lambda(N)=CN^{Q/2}\frac{\tau ^N}{N!},\quad N=0,1,\dots$$
and $\mathfrak{K}=\mathcal B_Q$. We recall also the condition \eqref{eq:epslowcond}.

\begin{proof}

For Chebyshev polynomials, by \cite[Section 12, Chapter 2]{Zygmund}, we have
$$\|S_n(F)\|_{L_\infty(I_\r)}\leq(\log n+1)^Q\|F\|_{L_\infty(I_\r)},\quad  n\in\NN.$$
Then for $F\in\bigoplus\limits_{j=0}^{n-1}X_j$,
$$\left\|\mathcal T_n(F)\right\|_{\Pi,1,n}=\sum\limits_{k=0}^{n-1}\left\|\mathsf{Proj}_k(F)\right\|_{L_\infty(I_\r)}\leq n(\log n+1)^Q\|F\|_{L_\infty(I_\r)},\quad n\geq1,$$
which means we can take $B_{n,1}=n(\log n+1)^Q$ for each $n\in\NN$. Now we only need to find a proper $N_2$ to apply Theorem \ref{theo:abstractentropy}.

Like in the proof of Theorem \ref{entropy analy}, our principle of choosing $N_2$ is finding it as large as we can under the restriction
$$\log\left(\frac{\Lambda(N_2)}{2(N_2+1)(\log N_2+1)^Q\epsilon}\right)\geq0.$$
To find a solution of
$$\frac{\Lambda(N)}{2(N+1)(\log N+1)^Q}\geq\epsilon,$$
we make the Stirling's estimation
\begin{equation*}
\begin{split}
\Lambda(N)\geq CN^{Q/2}\tau ^N\frac{1}{2\sqrt{2\pi N}}\left(\frac{e}{N}\right)^N=\frac{C(e\tau )^{N}}{2\sqrt{2\pi}}\left(\frac{1}{N}\right)^{N-\frac{Q-1}{2}}:=\Lambda_0(N).
\end{split}
\end{equation*}
Then
\begin{equation*}
\begin{split}
\frac{\Lambda(N)}{2(N+1)(\log N)^Q}\geq&\frac{\Lambda_0(N)}{2(N+1)(\log N)^Q}=\frac{C(e\tau)^{-\frac{1}{2}}}{2\sqrt{2\pi}}\left(\frac{e\tau}{N}\right)^{N+\frac{1}{2}}\left(\frac{N}{(\log N)^2}\right)^{Q/2}\frac{1}{N+1}\\
\geq&\frac{C}{4\sqrt{2\pi e\tau}}\left(\frac{e\tau}{N}\right)^{N+\frac{1}{2}}.
\end{split}
\end{equation*}
In this proof only, let
$$B=\log\left(\frac{C}{4\sqrt{2\pi e\tau}\epsilon}\right),$$
then it sufficies to find a solution of
$$e^B\left(\frac{e\tau }{N}\right)^{N+\frac{1}{2}}\geq1.$$
Taking logarithms on both sides, we conclude it suffices to solve
\be\label{solv N2 log}
\left(N+\frac{1}{2}\right)\log\frac{N}{e\tau }\leq B.
\ee
Let
\be\label{def:entire N2}
N_2=\left\lfloor\frac{ B}{\log B-\log(e\tau )}-\frac{1}{2}\right\rfloor.
\ee
Note $x\leq\frac{y}{\log y}\Rightarrow x\log x\leq y$ for all $y\geq e$, and it is clear that $\frac{B}{e\tau}\geq e$ under the condition that
$$\epsilon\leq \left(\frac{2\pi}{Q}\right)^{Q/2}\frac{1}{4\sqrt{2\pi e\tau}}\xi_\tau^{-2\xi_\tau}\leq\frac{1}{4\sqrt{2\pi e\tau}}\exp(-e^2\tau)\left(\frac{2\pi}{Q}\right)^{Q/2},$$
then
$$\frac{N_2+\frac{1}{2}}{e\tau }\log\frac{N_2+\frac{1}{2}}{e\tau }\leq\frac{ B}{e\tau }.$$
Consequently, $N_2$ is a solution of \eqref{solv N2 log}, hence, a solution of
$$\frac{\Lambda(N)}{2(N+1)(\log N)^Q}\geq\epsilon.$$
Now
$$\frac{\Lambda(N)}{2\epsilon(N_2+1)\log N_2}\geq\frac{\Lambda_0(N)}{2\epsilon(N_2+1)(\log N_2+1)^Q}=e^B\left(\frac{e\tau}{N}\right)^{N-\frac{Q-1}{2}}\left(\log N_2+1\right)^{-Q}$$
Together with Theorem \ref{theo:abstractentropy},
\begin{equation}
\begin{split}
&H_{\epsilon}\left(\mathcal B_{Q},\|\cdot\|_{L_\infty( I_\r)}\right)\geq\sum\limits_{N=0}^{N_2}\binom{N+Q-1}{Q-1}\log\left( e^B\left(\frac{e\tau }{N}\right)^{N-\frac{Q-1}{2}}\left(\log N_2+1\right)^{-Q}\right)\\
=&\sum\limits_{N=0}^{N_2}\binom{N+Q-1}{Q-1}\left[\log\left( e^B\left(\frac{e\tau }{N}\right)^{N+\frac{1}{2}}\right)+\frac{Q}{2}\log\frac{N}{e\tau}-Q\log\left(\log N_2+1\right)\right].
\end{split}
\end{equation}
On one hand,
\begin{equation*}
\begin{aligned}
\sum\limits_{N=0}^{N_2}\binom{N+Q-1}{Q-1}&\left[\frac{Q}{2}\log\frac{N}{e\tau}-Q\log\left(\log N_2+1\right)\right]\\
&\geq\sum\limits_{n=\lfloor N_2/2\rfloor}^{N_2}\binom{N+Q-1}{Q-1}\frac{Q}{2}\log\frac{N}{e\tau}-\binom{N_2+Q}{Q}Q\log\left(\log N_2+1\right)\\
\geq&\sum\limits_{n=\lfloor N_2/2\rfloor}^{N_2}\binom{N+Q-1}{Q-1}\frac{Q}{2}\log\frac{N_2}{3e\tau}-\binom{N_2+Q}{Q}Q\log\left(\log N_2+1\right)\\
\geq&\frac{1}{2}\binom{N_2+Q}{Q}\frac{Q}{2}\log\frac{N_2}{3e\tau}-\binom{N_2+Q}{Q}Q\log\left(\log N_2+1\right)\\
=&\frac{Q}{4}\binom{N_2+Q}{Q}\log\frac{N_2}{3e\tau\left(\log N_2+1\right)^4}.
\end{aligned}
\end{equation*}
In this proof only, let $a_\tau=\max\{128,3e^2\tau\}$. Since $\epsilon<\left(\frac{2\pi}{Q}\right)^{Q/2}\frac{1}{4\sqrt{2\pi e\tau}}\xi_\tau^{-2\xi_\tau}$, we have
$$B\geq2\frac{16a_\tau(\log a_\tau)^4+2e}{e}\log\frac{16a_\tau(\log a_\tau)^4+2e}{e}.$$
For $y\geq e$, we have $x\geq2y\log y\Rightarrow\frac{x}{\log x}\geq y$, so
$$\frac{B}{\log B-\log(e\tau)}\geq\frac{B}{\log B}\geq\frac{16a_\tau(\log a_\tau)^4+2e}{e}.$$
Therefore,
$$N_2\geq\frac{16a_\tau(\log a_\tau)^4}{e}.$$
We have also $x\geq16y(\log y)^4\Rightarrow\frac{x}{(\log x)^4}\geq y$ for $y\geq 128$, then
$$\frac{eN_2}{\log (eN_2)^4}\geq a_\tau\geq3e^2\tau,$$
this is
$$\frac{N_2}{3e\tau\left(\log N_2+1\right)^4}\geq1.$$
Consequently,
$$\sum\limits_{N=0}^{N_2}\binom{N+Q-1}{Q-1}\left[\frac{Q}{2}\log\frac{N}{e\tau}-Q\log\left(\log N_2+1\right)\right]\geq0.$$
On the other hand, since $ e^B\left(\frac{e\tau }{N_2}\right)^{N_2+\frac{1}{2}}\geq1$, we have for $N\leq\lfloor N_2/2\rfloor$,
\begin{equation*}
\begin{split}
 e^B\left(\frac{e\tau }{N}\right)^{N+\frac{1}{2}}\geq\left(\frac{e\tau }{N}\right)^{N+\frac{1}{2}}\left(\frac{N_2}{e\tau }\right)^{N_2+\frac{1}{2}}=\left(\frac{N_2}{N}\right)^{N+\frac{1}{2}}\left(\frac{N_2}{e\tau }\right)^{N_2-N}\geq\left(\frac{N_2}{e\tau }\right)^{\frac{N_2}{2}}.
\end{split}
\end{equation*}
Then
$$H_{\epsilon}\left(\mathcal B_{Q},\|\cdot\|_{L_\infty( I_\r)}\right)\geq\sum\limits_{N=0}^{\lfloor N_2/2\rfloor}\binom{N+Q-1}{Q-1}\frac{N_2}{2}(\log N_2-\log e\tau )+\frac{Q}{4}\binom{N_2+Q}{Q}\log\frac{N_2}{3e\tau\left(\log N_2+1\right)^4}.$$
By Lemma \ref{lemma:combinatorics},
\begin{equation}\label{entire lower bin}
H_{\epsilon}\left(\mathcal B_{Q},\|\cdot\|_{L_\infty( I_\r)}\right)\geq\frac{1}{2}\binom{\lfloor N_2/2\rfloor+Q}{Q}N_2(\log N_2-\log e\tau )+\frac{Q}{4}\binom{N_2+Q}{Q}\log\frac{N_2}{3e\tau\left(\log N_2+1\right)^4}.
\end{equation}

Next we express the bound \eqref{entire lower bin} in terms of $\epsilon$. To begin with, we see that
$$\left\lfloor\frac{N_2}{2}\right\rfloor\geq\frac{1}{2}N_2-\frac{1}{2}\geq\frac{1}{2}\left(\frac{B}{\log B-\log(e\tau)}-\frac{1}{2}-1\right)-\frac{1}{2}=\frac{1}{2}\frac{B}{\log B-\log(e\tau)}-\frac{5}{4}.$$
Apply Lemma \ref{lemma:asym comb} and substitude \eqref{def:entire N2},
\begin{equation*}
\begin{split}
&H_{\epsilon}\left(\mathcal B_{Q},\|\cdot\|_{L_\infty( I_\r)}\right)\geq\frac{1}{2}\frac{1}{8\sqrt{\pi Q}}\left(\frac{2(\lfloor N_2/2\rfloor+Q)}{Q}\right)^QN_2(\log N_2-\log e\tau )\\
\geq&\frac{1}{16\sqrt{\pi Q}}\left(\frac{2}{Q}\right)^Q\left(\frac{1}{2}\frac{ B}{\log B-\log(e\tau )}-\frac{5}{4}+Q\right)^Q\left(\frac{ B}{\log B-\log(e\tau )}-\frac{3}{2}\right)\\
\geq&\frac{1}{16\sqrt{\pi Q}}\left(\frac{1}{Q}\right)^Q\left(\frac{\log\frac{1 }{4\sqrt{2\pi e\tau}\epsilon}+\frac{Q}{2}\log\left(\frac{2\pi}{Q}\right)}{\log\left(\log\frac{1 }{4\sqrt{2\pi e\tau}\epsilon}+\frac{Q}{2}\log\left(\frac{2\pi}{Q}\right)\right)-\log(e\tau )}-\frac{5}{2}+Q\right)^Q\\
&\times\left(\frac{\log\frac{1 }{4\sqrt{2\pi e\tau}\epsilon}+\frac{Q}{2}\log\left(\frac{2\pi}{Q}\right)}{\log\left(\log\frac{1 }{4\sqrt{2\pi e\tau}\epsilon}+\frac{Q}{2}\log\left(\frac{2\pi}{Q}\right)\right)-\log(e\tau )}-\frac{3}{2}\right).
\end{split}
\end{equation*}

For the asymptotic relation of the bound when $\epsilon\to0$, a simple observation shows
$$N_2=\frac{\log\frac{1}{\epsilon}}{\log\log\frac{1}{\epsilon}}(1+o(1))$$
as $\epsilon\to0$. Similarly as before, \eqref{entire lower bin} gives
\begin{equation*}
\begin{split}
H_\epsilon\left(\mathcal B_Q,\|\cdot\|_{L_\infty( I_\r)}\right)\geq&\frac{1}{2}\frac{(\lfloor N_2/2\rfloor)^Q}{Q!}\frac{\log\frac{1}{\epsilon}}{\log\log\frac{1}{\epsilon}}\left(\log \frac{\log\frac{1}{\epsilon}}{\log\log\frac{1}{\epsilon}}\right)(1+o(1))\\
\geq&\frac{1}{2Q!}\left(\frac{\log\frac{1}{\epsilon}}{2\log\log\frac{1}{\epsilon}}\right)^Q\left(\frac{\log\frac{1}{\epsilon}}{\log\log\frac{1}{\epsilon}}\right)\left(\log\log\frac{1}{\epsilon}-\log\log\log\frac{1}{\epsilon}\right)(1+o(1))\\
\geq&\frac{1}{2Q!}\left(\frac{\log\frac{1}{\epsilon}}{2\log\log\frac{1}{\epsilon}}\right)^Q\left(\log\frac{1}{\epsilon}\right)(1+o(1)).
\end{split}
\end{equation*}
This proves \eqref{entropyB1}.

\end{proof}

\subsection{Proof of Theorem \ref{total entropy}}\label{pf:functional}

\begin{proof}

Since $\epsilon<\frac{\sqrt{1-\rho^2}}{4}\rho^q$, by Theorem \ref{entropy analy}, the $\epsilon/2$-entropy of $\mathcal A_\rho$ can be bounded by
\begin{equation}\label{total-anal}
H_{\epsilon/2}(\mathcal{A},\|\cdot\|_{L^2( I^q)})\leq\frac{4e^{q+1}}{\sqrt{2\pi}}\left(1+\frac{\log\left(\frac{2\rho}{\sqrt{1-\rho^2}}\frac{2}{\epsilon}\right)}{(q+1)\log\frac{1}{\rho}}\right)^{q+1}\leq\frac{4(2e)^{q+1}}{\sqrt{2\pi}}\left(\frac{\log\frac{1}{\epsilon}}{(q+1)\log\frac{1}{\rho}}\right)^{q+1}.
\end{equation}
Consider the upper bound of the $\epsilon/2$-entropy of $\mathcal F$. By taking $n$ as the integer $N_1$ in the proof of Theorem \ref{entropy analy},
$$n=\left\lfloor\frac{\log\frac{1}{\epsilon}+\log\frac{4}{\sqrt{1-\rho^2}}}{\log\frac{1}{\rho}}\right\rfloor,$$
we get from there that $\|f-s_{n+1}(f)\|_{L^2( I^q)}\leq\epsilon/4$ holds for all $f\in\mathcal A_\rho$.
Now $n$ is fixed in the rest of the proof. For convenience, denote $\tau =\sum\limits_{j=1}^{Q}v_jr_j$. Then $\tau =\frac{Q}{2e^{3/2}\pi}$.

In this case,
$$\left\|\tilde F-\tilde F\circ s_{n+1}\right\|=\sup\limits_{f\in\mathcal A_\rho}\left\|\tilde F(f)-\tilde F(s_{n+1}(f))\right\|\leq\sup\limits_{f\in\mathcal A_\rho}\|f-s_{n+1}(f)\|_{L^2( I^q)}\leq\epsilon/4.$$
Thus, any $\epsilon/4$-cover of the set $\{F\circ s_{n+1}:\ \tilde F\in\mathcal F\}\subset\mathcal F$ is an $\epsilon/2$-cover of $\mathcal F$.

The map $\tilde F\mapsto F$ denoted by
$$\tilde F(f)=F\left(\left(\hat f(\k)\right)_{|\k|_1\leq n}\right),\quad f\in\mathcal A_\rho$$
is an isometry from $\{F\circ s_{n+1}:\ \tilde F\in\mathcal F\}\subset\mathcal F$ to $\tilde{\mathcal B}_n$ with the $L_\infty(I_\r)$ norm. Therefore, for the entropy of the former, we only need to consider the entropy of $\tilde{\mathcal B}_n$.

Therefore, the $\epsilon$-entropy of $\mathcal F\times\mathcal A_\rho$ is bounded by
\begin{equation}
H_{\epsilon}(\mathcal F\times\mathcal A_\rho,\|\cdot\|)\leq H_{\epsilon/2}(\mathcal{A},\|\cdot\|_{L^2( I^q)})+H_{\epsilon/4}(\tilde{\mathcal B}_n,\|\cdot\|_{L_\infty(I_\r)}).
\end{equation}
Let 
$$\eta=\min\left\{\frac{\epsilon}{4},\left(\frac{2\pi e\tau }{Q}\right)^{Q/2}\frac{4}{(e\tau )^{1/2}\exp(e^2\tau )}\right\}.$$
Then using Theorem \ref{entropy analy} and Theorem \ref{entropy entire}, we conclude

\begin{equation}\label{entr by Q}
\begin{split}
&H_{\epsilon}(\mathcal F\times\mathcal A_\rho,\|\cdot\|)\leq\frac{4e^{q+1}}{\sqrt{2\pi}}\left(1+\frac{\log\left(\frac{2\rho}{\sqrt{1-\rho^2}}\frac{2}{\epsilon}\right)}{(q+1)\log\frac{1}{\rho}}\right)^{q+1}\\
&+\frac{2}{3\sqrt{2\pi}}\left(\frac{2e}{Q}\right)^Q\left(\frac{\log\frac{4}{\eta}+\frac{Q}{2}\log\frac{2e\pi \tau }{Q}}{\log\left(\log\frac{4}{\eta}+\frac{Q}{2}\log\frac{2e\pi \tau }{Q}\right)-\log(e\tau )}+\frac{3Q}{4}\right)^{Q+1}\left(5\log\log\frac{1}{\eta}+\log\left((Q+1)^2(e\tau)^6\right)\right)
\end{split}
\end{equation}
with $Q=\binom{n+q}{q}$.

Substituting $\tau =\frac{Q}{2e^{3/2}\pi}$ into $\eta$ and noticing that $Q\gg\log\frac{1}{\epsilon}$,
$$\eta=\min\left\{\frac{\epsilon}{4},e^{Q/4}4\left(\frac{2\sqrt e\pi}{Q}\right)^{1/2}\exp\left\{-\frac{\sqrt eQ}{2\pi}\right\}\right\}=4\left(\frac{2\sqrt e\pi}{Q}\right)^{1/2}\exp\left\{-Q\left(\frac{\sqrt e}{2\pi}-\frac{1}{4}\right)\right\}.$$
Then
\begin{equation*}
\begin{split}
\log\frac{4}{\eta}+\frac{Q}{2}\log\frac{2e\pi \tau }{Q}=&\log\left(\left(\frac{Q}{2\sqrt e\pi}\right)^{1/2}\exp\left\{Q\left(\frac{\sqrt e}{2\pi}-\frac{1}{4}\right)\right\}\left(\frac{2e\pi}{Q}\frac{Q}{2e^{3/2}\pi}\right)^{Q/2}\right)\\
=&\log\left(\left(\frac{Q}{2\sqrt e\pi}\right)^{1/2}\exp\left\{Q\frac{\sqrt e}{2\pi}\right\}\right)=e^2\tau +\frac{1}{2}\log(e\tau ).
\end{split}
\end{equation*}
Therefore, the latter term in \eqref{entr by Q} can be bounded as

\begin{equation*}
\begin{split}
&\frac{2}{3\sqrt{2\pi}}\left(\frac{2e}{Q}\right)^Q\left(\frac{\log\frac{4}{\eta}+\frac{Q}{2}\log\frac{2e\pi \tau }{Q}}{\log\left(\log\frac{4}{\eta}+\frac{Q}{2}\log\frac{2e\pi \tau }{Q}\right)-\log(e\tau )}+\frac{3Q}{4}\right)^{Q+1}\left(5\log\log\frac{1}{\eta}+\log\left((Q+1)^2(e\tau)^6\right)\right)\\
\leq&\frac{2}{3\sqrt{2\pi}}\left(\frac{2e}{Q}\right)^Q\left(e^2\tau +\frac{1}{2}\log(e\tau )+\frac{3}{4}Q\right)^{Q+1}\left(5\log\log Q+2\log\left(Q+1\right)+6\log(e\tau)\right)\\
\leq&\frac{2}{3\sqrt{2\pi}}\left(2e\left(\frac{\sqrt e}{2\pi}+1\right)\right)^Q\left(\frac{\sqrt e}{2\pi}+1\right)Q\left(5\log\log Q+8\log Q\right)\\
\leq&\frac{26}{3\sqrt{2\pi}}\left(2e\left(\frac{\sqrt e}{2\pi}+1\right)\right)^Q\left(\frac{\sqrt e}{2\pi}+1\right)Q\log Q.
\end{split}
\end{equation*}

The fact that $\epsilon<\frac{\sqrt{1-\rho^2}}{4}\rho^q$ implies
$$\log\frac{1}{\epsilon}+\log\frac{4}{\sqrt{1-\rho^2}}+q\log\frac{1}{\rho}\leq2\log\frac{1}{\epsilon}.$$
Then we can bound $Q$ by \eqref{bound bin} and get
\begin{equation*}
Q\leq\frac{2}{\sqrt{2\pi}}\left(\frac{e(L+q)}{q}\right)^q\leq e^q\left(\frac{\log\frac{1}{\epsilon}+\log\frac{4}{\sqrt{1-\rho^2}}}{q\log\frac{1}{\rho}}+1\right)^q\leq\left(\frac{2e\log\frac{1}{\epsilon}}{q\log\frac{1}{\rho}}\right)^q.
\end{equation*}
Consequently,
\begin{equation*}
\begin{split}
&\frac{2}{3\sqrt{2\pi}}\left(\frac{2e}{Q}\right)^Q\left(\frac{\log\frac{4}{\eta}+\frac{Q}{2}\log\frac{2e\pi \tau }{Q}}{\log\left(\log\frac{4}{\eta}+\frac{Q}{2}\log\frac{2e\pi \tau }{Q}\right)-\log(e\tau )}+\frac{3Q}{4}\right)^{Q+1}\left(5\log\log\frac{1}{\eta}+\log\left((Q+1)^2(e\tau)^6\right)\right)\\
\leq&\frac{26q}{3\sqrt{2\pi}}\exp\left\{\left(\frac{2e\log\frac{1}{\epsilon}}{q\log\frac{1}{\rho}}\right)^q\log\left(\frac{e^{3/2}}{\pi}+2e\right)\right\}\left(\frac{2e\log\frac{1}{\epsilon}}{q\log\frac{1}{\rho}}\right)^{q}\log\left(\frac{2e\log\frac{1}{\epsilon}}{q\log\frac{1}{\rho}}\right).
\end{split}
\end{equation*}

Combining this with \eqref{total-anal} and substituting the values of $C$ and $\tau $ into the inequality,
\begin{equation*}
\begin{split}
&H_{\epsilon}(\mathcal F\times\mathcal A_\rho,\|\cdot\|)\leq\frac{4(2e)^{q+1}}{\sqrt{2\pi}}\left(\frac{\log\frac{1}{\epsilon}}{(q+1)\log\frac{1}{\rho}}\right)^{q+1}\\
&+\frac{26q}{3\sqrt{2\pi}}\exp\left\{\left(\frac{2e\log\frac{1}{\epsilon}}{q\log\frac{1}{\rho}}\right)^q\log\left(\frac{e^{3/2}}{\pi}+2e\right)\right\}\left(\frac{2e\log\frac{1}{\epsilon}}{q\log\frac{1}{\rho}}\right)^{q}\log\left(\frac{2e\log\frac{1}{\epsilon}}{q\log\frac{1}{\rho}}\right)\\
\leq&\frac{4}{\sqrt{2\pi}}\gamma^{q+1}\left(\frac{q}{q+1}\right)^{q+1}+\frac{26q}{3\sqrt{2\pi}}\exp\left\{\gamma^q\log\left(\frac{e^{3/2}}{\pi}+2e\right)\right\}\gamma^{q}\log\gamma\\
\leq&\gamma^{q+1}\left(\frac{q}{q+1}\right)^{q+1}+\frac{26q}{3\sqrt{2\pi}}\exp\left\{\gamma^q\log\left(\frac{e^{3/2}}{\pi}+2e\right)\right\}\gamma^{q}\log\gamma,
\end{split}
\end{equation*}
where
$$\gamma=\frac{2e\log\frac{1}{\epsilon}}{q\log\frac{1}{\rho}}.$$
This completes the proof.

\end{proof}

\bhag{Conclusions}\label{bhag:conclude}

We studied the question of which algorithms and data sets are close to each other in terms of some performance metrics. 
The problem was formulated in a mathematically rigorous manner as the one of finding an optimal $\epsilon$-net for a tensor product of two (infinite dimensional) sets: one representing the data sets and one the algorithms.
We solved this problem under certain simplifying assumptions, with details included in the attached preliminary note.

\appendix
\renewcommand{\theequation}{\Alph{section}.\hindu{equation}}

\section{Degree of approximation}\label{bhag:degapprox}
\subsection{Analytic functions}\label{sec:analdeg}

The following lemma is a straightforward consequence of the corresponding well-known one-dimensional results.
\begin{lemma}\label{lemma:complexfacts}
{\rm (a)} If  $\bs r>0$, $\k\in\NN^d$ then for $\z\in\CC^d\setminus I_{\bs r}$, we have
\be\label{eq:bernsteinwalsh}
|p_{\k,\r}(\z)|\le (\bs r)^{-\k}\prod_{j=1}^d \left|z_j+\sqrt{z_j^2-r_j^2}\right|^{k_j}.
\ee
{\rm (b)} If $0<\bs\rho<1$, $f$ is analytic on the closure of $U_{\bs\rho}$, $\partial \Gamma_{\bs\rho}$ is the boundary of $ \Gamma_{\bs\rho}$ and $g(\w)=f((\w+\w^{-1})/2)$, then
\be\label{eq:chebaslaurent}
\hat{f}(\k)=\frac{1}{(2\pi i)^d}\int_{\partial  \Gamma_{\bs \rho}} \frac{g(\w)}{\w^{\k+1}}d\w.
\ee
In particular,
\be\label{eq:chebcoeffest}
\left|\hat{f}(\k)\right|\le (\bs\rho)^{\k}\max_{\w\in \partial  \Gamma_{\bs\rho}}|g(\w)|.
\ee
\end{lemma}

\subsection{Analytic functions}

\begin{theorem}\label{theo:identifyanal}
Let $S_n$, $n=0,1,\dots$ be the operators denoted in \eqref{eq:partialsum}. A function $f$ is analytic on $U_\rho$ if and only if
\be\label{anal cond}
\lim\sup\limits_{n\to\infty}\left\|S_{n}(f)\right\|_{L^2( I^q)}^{\frac{1}{n}}\le \rho.
\ee
\end{theorem}

\begin{proof}[Proof of Theorem \ref{theo:identifyanal}]

Suppose $f$ is analytic on $U_{\rho}$, then $f$ is analytic on the closure of $U_{1/(\rho+\eta)}$. By \eqref{eq:chebcoeffest} and \eqref{eq:chebaslaurent},
$$\|S_{n}(f)\|_{L^2(I^q)}=\left(\sum\limits_{|\k|_1=n}\left|\hat{f}(\k)\right|^2\right)^{1/2}\leq\sqrt{\binom{n+q-1}{q-1}}(\rho+\eta)^{|\k|_1}\max_{\z\in U_{1/(\rho+\eta)}}|f(\z)|\leq C(\eta)(\rho+2\eta)^n.$$
Thus,
$$
\limsup_{n\to\infty}\|S_n(f)\|_{L^2(I^q)}^{1/n}\le \rho.
$$

Now suppose $\|S_{n}(f)\|_{L^2(I^q)}\leq\rho^n$ for all $n\in\NN$, then
$$\left|\hat{f}(\k)\right|\leq\left(\sum\limits_{|\j|_1=|\k|_1}\left|\hat{f}(\j)\right|^2\right)^{1/2}=\left\|S_{|\k|_1}(f)\right\|_{L^2(I^q)}\leq\rho^{|\k|_1},\quad\k\in\NN.$$
For any $\z\in U_{1/\rho}$, let $\rho':=\left(\max\limits_{1\leq j\leq q}\left|z_j+\sqrt{z_j^2-1}\right|\right)^{-1}>\rho$, then \eqref{eq:bernsteinwalsh} implies that
$$|p_\k(\z)|\leq\max_{\w\in \partial  \Gamma_{\bs\rho}}|g(\w)|\prod_{j=1}^q \left|z_j+\sqrt{z_j^2-1}\right|^{k_j}\leq\prod_{j=1}^q (1/\rho')^{k_j}.$$
Together with \eqref{eq:chebcoeffest},
$$\sum\limits_{\k\in\NN^q}\left|\hat{f}(\k)p_\k(\z)\right|\leq\max_{\w\in \partial U_{\rho}}|g(\w)|\sum\limits_{n=0}^\infty\sum\limits_{|\k|_1=n}\frac{\rho^n}{\rho'^n}=\max_{\w\in \partial  \Gamma_{\rho}}|g(\w)|\sum\limits_{n=0}^\infty\binom{n+q-1}{q-1}\left(\frac{\rho}{\rho'}\right)^n<\infty.$$
Hence, $f$ is analytic at $\z$, which implies $f$ is analytic on $U_{\rho}$.

\end{proof}

\subsection{Entire functions}\label{sec:entiredeg}

\begin{theorem}\label{thm entire}
Let $Q\in\NN$, $\mathbf{v}=(v_1,\dots,v_Q),\ \mathbf{r}=(r_1,\dots,r_Q)\in\RR_+^Q$, $I_\r=\prod\limits_{j=1}^Q[-r_j,r_j]\subset\RR^Q$ and $\left\{p_{\k,\r}\right\}_{\k\in\NN^Q}$ the multivariable Chebyshev polynomials orthonormal on $I_\r$. Let $F:\ \CC^Q\to\CC$ be an entire function with
\begin{equation}
\sup\limits_{\z\in\CC^Q}\left|F(\z)\right|\leq\exp\left\{\sum\limits_{j=1}^Qv_j|z_j|\right\},
\end{equation}
then
\begin{equation}\label{Linfty for entire}
\begin{split}
\left\|\sum\limits_{|\k|_1=N}\hat F_\r(\k)p_{\k,\r}\right\|_{L_\infty(I_\r)}\leq2\left(\frac{2\pi}{Q}\right)^{Q/2}N^{Q/2}\frac{(\v\cdot\r)^N}{N!}.
\end{split}
\end{equation}
\end{theorem}

Conversely, if $F$ is a function on $\CC^Q$ satisfying \eqref{Linfty for entire} for each $\bs r\in\RR_+^Q$, then we can prove it is an entire function.

\begin{theorem}\label{converse}
Let $Q\in\NN$, $\mathbf{v}=(v_1,\dots,v_Q)\in\RR_+^Q$. If a function $F:\ \CC^Q\to\CC$ satisfies \eqref{Linfty for entire} for any $\r\in\RR^Q$, then for any $\mathbf{z}\in\CC^Q$,
\begin{equation}\label{entire converse}
\sup\limits_{\mathbf{z}\in\CC^Q}\left(\frac{|F(\mathbf{z})|}{\exp\left\{2\left(\sum\limits_{j=1}^Qv_j|z_j|\right)(1+\eta)\right\}}\right)<\infty,\quad\forall\eta>0.
\end{equation}

\end{theorem}

\begin{proof}[Proof of Theorem \ref{thm entire}]
First, we consider $I_\r=I^q$ as the unit cube. In this case, we write $p_{\k,\r}$ as $p_\k$ for $\k\in\NN^Q$.

By \eqref{eq:chebaslaurent}, 
we have
$$\left|\hat{F}(\k)\right|\le \frac{1}{\k^\k}\max_{\z\in U_{\bs\rho}}|F(\z)|\leq \frac{e^{|\k|_1}\v^\k}{\k^\k}.$$
By Stirling's approximation,
$$\left|\hat{F}(\k)\right|\leq 2\frac{(\sqrt{2\pi})^Q(k_1\dots k_Q)^{1/2}(\v+\bs\eta)^\k}{\k!}$$
for some $A'_\eta$ depending only on $\eta$ and $Q$.

Observing that
\begin{equation}\label{eqn binom}
(y_1+\dots+y_Q)^N=\sum\limits_{\|\mathbf{k}\|_1=N}\frac{N!}{{k_1!k_2!\dots k_Q!}}\prod\limits_{j=1}^Qy_j^{k_j},
\end{equation}
we have
$$\|S_N(F)\|_{L_\infty(I^Q)}=\sup\limits_{\x\in I^q}\left|\sum\limits_{|\k|_1=N}\hat F(\k)P_{\k}(\x)\right|\leq\sum\limits_{|\k|_1=N}\left|\hat F(\k)\right|\leq 2\left(\frac{2\pi}{Q}\right)^{Q/2}\frac{N^{Q/2}}{N!}|\v|_1^N.$$

Now we make a change of variables. For $\r\in\RR_+^Q$, let $G(\z)=F(\z\r)$, then $G$ is an entire function with
$$\sup\limits_{\z\in\CC^Q}\left|G(\z)\right|\leq \exp\left\{\sum\limits_{j=1}^Qv_jr_j|z_j|\right\}.$$

Hence,
$$\left\|\sum\limits_{|\k|_1=N}\hat G(\k)p_\k\right\|_{L_\infty(I^q)}\leq 2\left(\frac{2\pi}{Q}\right)^{Q/2}N^{Q/2}\frac{(\v\cdot\r)^N}{N!},$$
where
$$\hat G(\k)=\int_{I^q}G(\x)p_\k(\x)d\x=\int_{I^q}F(\r\circ\x)p_\k(\x)v_Q(\x)d\x=\left(\prod\limits_{j=0}^Qr_j\right)\int_{I_\r}F(\y)p_\k\left(\frac{y_1}{r_1},\dots,\frac{y_Q}{r_Q}\right)v_{Q,\r}\left(\y\right)d\y.$$
It is known that $\left\{p_{\k,\r}\right\}_{\k\in\NN^Q}$ is denoted by
$$p_{\k,\r}(\y)=p_\k\left(\frac{y_1}{r_1},\dots,\frac{y_Q}{r_Q}\right),\quad y\in I_\r,$$
hence $\hat F_\r(\k)=\hat G(\k)$. Therefore,
$$\|S_N(F)\|_{L_\infty(I_\r)}=\left\|\sum\limits_{|\k|_1=N}\hat F_\r(\k)p_{\k,\r}\right\|_{L_\infty(I_\r)}\leq\left\|\sum\limits_{|\k|_1=N}\hat G(\k)p_\k\right\|_{L_\infty({I^q})}\leq 2\left(\frac{2\pi}{Q}\right)^{Q/2}N^{Q/2}\frac{(\v\cdot\r)^N}{N!}.$$
This proves Theorem~ \ref{thm entire}. 
\end{proof}

\begin{proof}[Proof of Theorem \ref{converse}]
Suppose
$$\|S_N(F)\|_{L_\infty(I_\r)}\leq2\left(\frac{2\pi}{Q}\right)^{Q/2}N^{Q/2}\frac{(\v\cdot\r)^N}{N!}$$ holds true for every $\r\in\RR^Q$.

For any $\z\in\CC^Q$, take $\mathbf{r}=(|z_1|,\dots,|z_Q|)$. Since
$$\sup\limits_{|z|=1,z\in\CC}\frac{\left|z+\sqrt{z^2-1}\right|}{|z|}\leq2,$$
by \eqref{eq:bernsteinwalsh},
$$
|p_{\k,\r}(\z)|\leq 2^N\|p_{\k,\r}\|_{L_\infty(I_\r)},\quad\k\in\NN^Q.
$$

With $C=2\left(\frac{2\pi}{Q}\right)^{Q/2}$,
\begin{eqnarray*}
\sum\limits_{|\k|_1=N}\left|\hat F(\k)p_{\k,\r}(\z)\right|\leq CN^{Q/2}\frac{\left(2\mathbf{v}\cdot\mathbf{r}\right)^N}{N!}.
\end{eqnarray*}
We will use Stirling's approximation to eliminate the $N^{Q/2}$ term. By Stirling's approximation, for $N\geq Q/2$,
\begin{eqnarray*}
\sum\limits_{|\k|_1=N}\left|\hat F(\k)p_{\k,\r}(\z)\right|&\leq&\frac{C}{\sqrt{2\pi N}}N^{Q/2}\left(\frac{e}{N}\right)^N\left(2\mathbf{v}\cdot\mathbf{r}\right)^N\leq\frac{C}{\sqrt{2\pi N}}\left(\frac{e}{N-Q/2}\right)^{N-Q/2}e^{Q/2}\left(2\mathbf{v}\cdot\mathbf{r}\right)^N\\
&\leq&\frac{C}{\sqrt{2\pi N}}\frac{2\sqrt{2\pi(N-Q/2)}}{(N-Q/2)!}e^{Q/2}\left(2\mathbf{v}\cdot\mathbf{r}\right)^N\\
&\leq&2\left(\frac{2\pi}{Q}\right)^{Q/2}\left(e\mathbf{v}\cdot\mathbf{r}\right)^{Q/2}\frac{\left(2\mathbf{v}\cdot\mathbf{r}\right)^{N-Q/2}}{(N-Q/2)!}.
\end{eqnarray*}
Therefore for $\r\in\RR_{\geq1}^Q$, $\sum\limits_{|\k|_1< Q/2}\left|\hat F(\k)p_{\k,\r}(\z)\right|$ is bounded by $\mathcal P_1\left(2\mathbf{v}\cdot\mathbf{r}\right)$ with $\mathcal P_1$ a polynomial of degree $Q/2-1$ and
\begin{eqnarray*}
\sum\limits_{|\k|_1\geq Q/2}\left|\hat F(\k)p_{\k,\r}(\z)\right|&\leq&\sum\limits_{N=Q/2}^{\infty}2C\left(e\mathbf{v}\cdot\mathbf{r}\right)^{Q/2}\frac{\left(2\mathbf{v}\cdot\mathbf{r}\right)^{N-Q/2}}{(N-Q/2)!}\\
&=&\sum\limits_{N=0}^{\infty}2C\left(e\mathbf{v}\cdot\mathbf{r}\right)^{Q/2}\frac{\left(2\mathbf{v}\cdot\mathbf{r}\right)^{N}}{N!}\\
&=&2C\left(e\mathbf{v}\cdot\mathbf{r}\right)^{Q/2}\exp\left(2\mathbf{v}\cdot\mathbf{r}\right)\\
&\leq&4\left(\frac{2\pi}{Q}\right)^{Q/2}\left(e\mathbf{v}\cdot\mathbf{r}\right)^{Q/2}\exp\left(2\mathbf{v}\cdot\mathbf{r}\right).
\end{eqnarray*}
Now we can bound $F(\z)$ by
\begin{eqnarray*}
|F(\z)|\leq\sum\limits_{|\k|_1< Q/2}\left|\hat F(\k)p_{\k,\r}(\z)\right|+\sum\limits_{|\k|_1\geq Q/2}\left|\hat F(\k)p_{\k,\r}(\z)\right|\leq\mathcal P_1\left(2\mathbf{v}\cdot\mathbf{r}\right)+\mathcal P_2\left(2\mathbf{v}\cdot\mathbf{r}\right)\exp\left(2\mathbf{v}\cdot\mathbf{r}\right)
\end{eqnarray*}
with $\mathcal P_2(x)=4\left(\frac{2\pi}{Q}\right)^{Q/2}\left(ex\right)^{Q/2}$.

Since $\mathcal P_1$ and $\mathcal P_2$ are polynomials of degree at most $Q$, we conclude for any $\eta>0$, there exists some constant $A_\eta$ depending on $\eta$ and $Q$ such that
\begin{eqnarray*}
|F(\z)|\leq A_\eta\exp\left(2(\mathbf{v}+\bs\eta)\cdot\mathbf{r}(1+\eta)\right)=A_\eta\exp\left(\left(2\sum\limits_{j=1}^Q(v_j+\eta)|z_j|\right)(1+\eta)\right).
\end{eqnarray*}

\end{proof}

\bibliographystyle{abbrv}
\bibliography{references2}
\end{document}